\documentclass{article}

\usepackage{microtype}
\usepackage{graphicx}
\usepackage{subfigure}
\usepackage{booktabs} 

\usepackage[accepted]{icml2024}

\usepackage{amsmath}
\usepackage{amssymb}
\usepackage{mathtools}
\usepackage{amsthm}

\usepackage[capitalize,noabbrev]{cleveref}

\theoremstyle{plain}
\newtheorem{theorem}{Theorem}[section]
\newtheorem{proposition}[theorem]{Proposition}
\newtheorem{lemma}[theorem]{Lemma}
\newtheorem{corollary}[theorem]{Corollary}
\theoremstyle{definition}
\newtheorem{definition}[theorem]{Definition}
\newtheorem{assumption}[theorem]{Assumption}
\theoremstyle{remark}

\usepackage[textsize=tiny]{todonotes}

\usepackage[most]{tcolorbox}
\DeclareMathOperator*{\argmax}{arg\,max}
\DeclareMathOperator*{\argmin}{arg\,min}
\usepackage{graphicx}
\newcommand{\hl}[1]{\textcolor{black}{#1}}

\icmltitlerunning{Pausing \hl{Policy} Learning in \hl{Non-stationary} Reinforcement Learning}

\begin{document}

\twocolumn[
\icmltitle{Pausing \hl{Policy} Learning in \hl{Non-stationary} Reinforcement Learning}



\icmlsetsymbol{equal}{*}

\begin{icmlauthorlist}
\icmlauthor{Hyunin Lee}{ucb}
\icmlauthor{Ming Jin}{vt}
\icmlauthor{Javad Lavaei}{ucb}
\icmlauthor{Somayeh Sojoudi}{ucb}

\end{icmlauthorlist}

\icmlaffiliation{ucb}{University of California, Berkeley}
\icmlaffiliation{vt}{Virginia Tech}

\icmlcorrespondingauthor{Hyunin Lee}{hyunin@berkeley.edu}

\icmlkeywords{Machine Learning, ICML}

\vskip 0.3in
]



\printAffiliationsAndNotice{}  

\begin{abstract}
    Real-time inference is a challenge of real-world reinforcement learning due to temporal differences in time-varying environments: the system collects data from the past, updates the decision model in the present, and deploys it in the future. We tackle a common belief that continually updating the decision is optimal to minimize the temporal gap. We propose forecasting an online reinforcement learning framework and show that strategically pausing decision updates yields better overall performance by effectively managing aleatoric uncertainty. Theoretically, we compute an optimal ratio between policy update and hold duration, and show that a non-zero policy hold duration provides a sharper upper bound on the dynamic regret. Our experimental evaluations on three different environments also reveal that a non-zero policy hold duration yields higher rewards compared to continuous decision updates.
\end{abstract}

\section{Introduction}
Real-world reinforcement learning (RL) bridges the gap between the current literature on RL and real-world problems. \emph{Real-time inference}, a key challenge in real-world RL, requires that inference occur in real-time at the control frequency of the system \cite{dulac2019challenges}. For RL deployment in a production system, policy inference must occur in real-time, matching the control frequency of the system. This could range from milliseconds for tasks such as recommendation systems \cite{paul2016,steck2021deep} or autonomous vehicle control \cite{hester2013texplore}, to minutes for building control systems \cite{evansgao}. This constraint prevents us from speeding up the task beyond real-time to rapidly generate extensive data \cite{silver2016,espeholt2018impala} or slowing it down for more computationally intensive approaches \cite{levine2019prediction,schrittwieser2020mastering}. One strategy for real-time action is to employ a multi-threaded architecture, where model learning and planning occur in background threads while actions are returned in real-time \cite{hester2013texplore,Imanberdiyev2016AutonomousNO,Glavic2017ReinforcementLF}. 

In this paper, we show that intentionally pausing \hl{policy} learning can lead to better overall performance than continuous \hl{policy} updating. Our study is based on deriving an analytical solution for the optimal ratio between the pausing and updating phases. Perhaps most importantly, this paper offers the insight that the pausing phase is crucial to handling an aleatoric uncertainty that stems from the environment's intrinsic uncertainty.  

This paper begins with a fundamental observation of the real-time inference mechanism based on prediction: the agent forecasts the \textit{future} based on \textit{past} data, and then continually updates decisions in the \textit{present} based on future predictions. This highlights the significance of balancing conservatism or pessimism in decision-making, based on the three types of uncertainties: epistemic, aleatoric, and predictive uncertainties \cite{Gal2016UncertaintyID}. We define conservatism as expecting past trends to continue in the future, and pessimism as anticipating future differences. Although accumulating extensive past data reduces aleatoric uncertainty, and a prediction model with high capacity lessens predictive uncertainty, the frequency of policy updates still remains a key factor due to unknown aleatoric uncertainty in the present. 

\begin{figure}[ht]
\centering
    \subfigure[]{\includegraphics[width=0.16\textwidth]{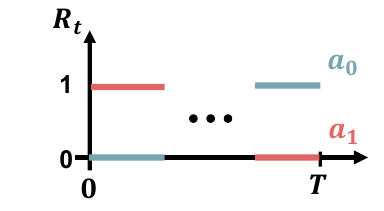}}
    \subfigure[]{\includegraphics[width=0.15\textwidth]{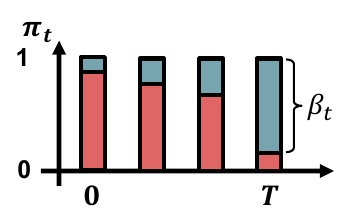}}
    \subfigure[]{\includegraphics[width=0.15\textwidth]{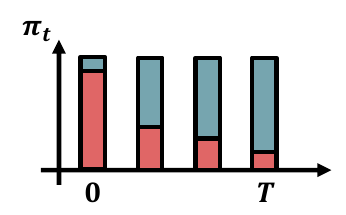}}
    \caption{(a) Non-stationary bandit setting, (b) conservative policy, (c) pessimistic policy}
    \label{fig:motivation}
\end{figure}

To elucidate the importance of the above problem, consider a recommendation system tasked with optimally suggesting item \( x_0 \) or \( x_1 \) to a user whose preference changes over time. This can be framed as a Bernoulli non-stationary bandit setting with a set of two actions \( \mathcal{A} = \{a_0 , a_1\} \), and a time-dependent policy \( \pi_t : \mathcal{A} \rightarrow [0,1] \), where \( \pi_t(a_0) = \beta_t \) and \( \pi_t(a_1) = 1 - \beta_t \), $ 0 \leq \beta_t \leq 1$. The rewards of each action, denoted as \( R_t \), switch (i.e., \( R_t(a_0) \leftrightarrow R_t(a_1) \)) once at an unpredictable time between \( 0 \) and \( T \) (see Figure \ref{fig:motivation} (a)). The goal of the system is to maximize the average rewards over a period \( T \), i.e., $\max_{\pi_1,...,\pi_T} \mathbb{E} \left[ \sum_{t=0}^T  R_t (a) \right]$. Initially, recommending \( x_1 \) yields a higher reward ($R_0(a_1)=1$). However, the system anticipates a shift in the user preference towards \( x_0 \) by the end of period \( T \). The system should optimize its policy \( \pi_t \) during the interval from \( 0 \) to \( T \), facing aleatoric uncertainty about when the user preferences will change. A conservative policy increases the preference weight \( \beta_t \) associated with $x_0$ too quickly (Figure \ref{fig:motivation} (b)), while a pessimistic approach may adjust too slowly (Figure \ref{fig:motivation} (c)). The key challenge is to determine the optimal tempo of policy adjustment in anticipation of this unknown preference shift.

Based on the previous example, this paper challenges the belief that continually updating the decision always achieves an optimal bound of dynamic regret, a measurement of decision optimality in a time-varying environment. Our main contribution, Algorithm \ref{algo1} and Theorem \ref{theorem_optimalGN}, demonstrates that strategically pausing decision updates provides a sharper upper bound on the dynamic regret by deriving an optimal ratio between the policy update duration and the pause duration. 

To achieve this, we formulate the online interactive learning problem in Section \ref{problem_Statement} by determining three key aspects: 1) the frequency of policy updates, 2) the timing of policy updates, and 3) the extent of each update. First, we study the real-time inference mechanism by proposing a forecasting online reinforcement learning model-free framework in Section \ref{Method}. In Section \ref{Theoretical Analysis}, we calculate an upper bound on the dynamic regret (Theorem \ref{theorem1}) as a function of episodic and predictive uncertainties (Propositions \ref{prop1} and \ref{prop2}), as well as aleatoric uncertainty (Proposition \ref{proposition:R_env} and Lemma \ref{lemma:R_env}). This is achieved by separating it into the policy update phase (Lemma \ref{lemma1}) and the policy hold phase (Lemma \ref{lemma2}). In Subsection \ref{Numerical verification on theoretical insights}, we conduct numerical experiments to show how the optimal ratio minimizing the dynamic regret's upper bound (Theorem \ref{theorem_optimalGN}) varies with hyperparameters related to aleatoric uncertainty, highlighting the significance of the policy hold phase in this minimization. Finally, in Section \ref{Experiments}, we empirically show two findings from three non-stationary environments: 1) a higher average reward of the forecasting method compared to the reactive method (Subsection \ref{Goal switching cliffworld}), and 2) a non-positive correlation relationship between update ratios and average returns (Subsection \ref{Mujoco environment}).

\subsection*{Notations}
The sets of natural, real, and non-negative real numbers are denoted by \(\mathbb{N}\), \(\mathbb{R}\), and \(\mathbb{R}_+\), respectively. For a finite set \(Z\), the notation \(|Z|\) represents its cardinality, and \(\Delta(Z)\) denotes the probability simplex over \(Z\). Given \(X, Y \in \mathbb{N}\) with \(X < Y\), we define \([X] := \{1, 2, \ldots, X\}\), the closed interval \([X, Y] := \{X, X+1, \ldots, Y\}\), and the half-open interval \([X, Y) := \{X, X+1, \ldots, Y-1\}\). For \(x \in \mathbb{R}_+\), the floor function \(\lfloor x \rfloor\) is defined as \(\max \{ n \in \mathbb{N} \cup \{0\} \mid n \leq x \}\). For any functions \(f, g : \mathbb{R}^m \rightarrow \mathbb{R}\) satisfying \(f(x) \leq g(x)\) for all values of $x$, if \(x^*_{g} = \argmin_{x \in \mathbb{R}^m} g(x)\), then \(x^*_{g}\) is referred to as a surrogate optimal solution of \(f(x)\). We use the term surrogate optimal solution and suboptimal solution interchangeably.

\section{Related works}
\subsection*{Real-time inference RL} 

One approach to real-time reinforcement learning is to adapt existing algorithms and validate their feasibility for real-time operation \cite{adam2012}. Alternatively, some algorithms are specifically designed with the primary objective of functioning in real-time contexts \cite{cai2017real,Wang2015RealTimeBA}. A recent and distinct perspective on real-time inference was presented in \cite{ramstedt2019real}, which proposed a real-time markov reward process. In this process, the state evolves concurrently with the action selection. The anytime inference approach \cite{vlasselaer2015anytime,pmlr-vR3-spirtes01a} encompasses a set of algorithms capable of returning a valid solution at any interruption point, with their performance improving over time. 



\subsection*{Non-stationary RL}

The problem formulation of this paper draws inspiration from ``desynchronized-time environment", initially proposed by \cite{lee2023tempo}. The desynchronized-time environment assigns the real-time duration of the learning process, where the agent is responsible for deciding both the timing and the duration of its interactions. \cite{finn2019online} introduced the Follow-The-Meta-Leader algorithm to improve parameter initialization in a non-stationary environment, but it cannot efficiently handle delays in optimal policy tracking. To address this, \cite{chandak2020optimizing,chandak2020towards} developed methods for forecasting policy evaluation, yet faced limitations in empirical analysis and theoretical bounds for policy performance. \cite{pmlr-v139-mao21b} proposed an adaptive $Q$-learning approach with a restart strategy, establishing a near-optimal dynamic regret bound. 

\hl{We will further elaborate on related work on non-stationary RL in Appendix \ref{appendix:Related works}.}


\section{Problem Statement}
\label{problem_Statement}
\textbf{Time-elapsing Markov Decision Process }\cite{lee2023tempo}.
For a given time $t \in [0,T]$, we define the Markov Decision Process (MDP) at time $t$ as $\mathcal{M}_t \coloneqq \langle \mathcal{S},\mathcal{A},P_t,R_t,\gamma,H \rangle $. $\mathcal{S}$ is a state space, $\mathcal{A}$ is an action space, $P_t : \mathcal{S} \times \mathcal{A} \times \mathcal{S} \rightarrow \Delta (\mathcal{S})$ is a transition probability at time $t$, and $R_t : \mathcal{S} \times \mathcal{A} \rightarrow \mathbb{R}$ is a reward function at time $t$. For every time $t$, the agent interacts with the environment via a policy $\pi_t : \mathcal{S} \times \mathcal{A} \rightarrow \Delta (\mathcal{S})$ where each episode takes $H$ steps to complete. We assume that a trajectory is finished within a second, implying that the agent will finish its trajectory within a temporally fixed MDP $\mathcal{M}_t$.

\textbf{Time elapsing variation budget.}
In the real world, the time of the environment flows independently from $t=0$ to $t=T$ regardless of the agent's behavior. For any time instances $t_1,t_2 \in [0,T)$ such that $t_1 < t_2$, we define \emph{local variation budgets} $B_r (t_1,t_2)$ and $B_p (t_1,t_2)$ as 
\begin{align*}
    B_r (t_1,t_2) &\coloneqq \sum_{t=t_1}^{t_2-1} \max_{s,a} \left| R_{t+1}(s,a) - R_{t}(s,a) \right|, \\
    B_p (t_1,t_2) &\coloneqq \sum_{t=t_1}^{t_2-1} \max_{s,a}  \left| \left| P_{t+1}(\cdot~|~s,a) - P_{t}(\cdot~|~s,a) \right| \right|_{1}.
\end{align*}
Also, we define \emph{cumulative variation budgets} $\bar{B}_p (t_1,t_2)$ and $\bar{B}_r (t_1,t_2)$ as the summation of local variation budgets between time $t_1$ and $t_2$, i.e., 
$$\bar{B}_r (t_1,t_2)\coloneqq \sum_{t=t_1}^{t_2-1} B_r(t_1,t), 
    \bar{B}_p (t_1,t_2) \coloneqq \sum_{t=t_1}^{t_2-1} B_p(t_1,t).$$
To align with real-world scenarios where environmental changes do not normally occur too abruptly, we propose that these changes follow an exponential growth.
\begin{assumption}[Exponential order local variation budget]
    For any time interval $[t_1,t_2] \subset [0,T)$, there exist constants $k_r,k_p > 1, B^{\text{max}}_p,B^{\text{max}}_r >0$ such that $B_p(t_1,t) \leq B^{\text{max}}_p k_{p}^{t-t_1}$ and $B_r(t_1,t) \leq B^{\text{max}}_r k_{r}^{t-t_1}$ hold for $\forall t \in [t_1,t_2]$ .
    \label{assum:Exponential order local variation budget}
\end{assumption}
Building on Assumption \ref{assum:Exponential order local variation budget}, we will derive cumulative variation budgets that also adhere to an exponential order.
\begin{corollary}[Exponential order cumulative variation budget]
    For arbitrary time instances $t_1,t_2 \in [0,T)$ satisfying $t_1<t_2$, there exist constants $\alpha_r,\alpha_p > 1$ such that $\bar{B}_p(t_1,t_2) \leq B^{\text{max}}_p \alpha_{p}^{t_2-t_1}$ and $\bar{B}_r(t_1,t_2) \leq B^{\text{max}}_r \alpha_{r}^{t_2-t_1}$ hold.
    \label{cor:Exponential order cumulative variation budget}
\end{corollary}
Next, we define stationary and non-stationary environments in the context of variation budget. 
\begin{definition}[Stationary environment]
    For arbitrary time instances $t_1,t_2 \in [0,T]$, if $B_r(t_1,t_2)=0$ and $B_p(t_1,t_2)=0$ are satisfied, then we call the corresponding environment a stationary environment.
    \label{def:stationary}
\end{definition}
\begin{definition}[Non-stationary environment]
    If there exist $t_1,t_2 \in [0,T]$ such that $B_r(t_1,t_2)>0$ or $B_p(t_1,t_2)>0$, then we call the corresponding environment a non-stationary environment. 
    \label{def:nonstationary}
\end{definition}
\textbf{State value function, State action value function. } For any policy $\pi$, we define the state value function $V^{\pi}_{t} : \mathcal{S} \rightarrow \mathbb{R} $ and the state action value function $Q^{\pi}_{t} : \mathcal{S} \times \mathcal{A}  \rightarrow \mathbb{R} $ at time $t$ as  $V^{\pi}_{t}(s) := \mathbb{E}_{\mathcal{M}_t} \left[ \sum_{h=0}^{H-1} \gamma^{h} r_{t,h}~|~s^{0}_{t} = s \right]$ and $ Q^{\pi}_{t}(s,a) := \mathbb{E}_{\mathcal{M}_t} \left[ \sum_{h=0}^{H-1} \gamma^{h} r_{t,h} ~|~s^{0}_{t} = s,a^{0}_{t} = a \right]$, where $r_{t,h} := R_{t}(s^{h}_{t},a^{h}_{t})$. We define the optimal policy at time $t$ as $\pi^{*}_{t}=\argmax_{\pi} V^{\pi}_{t}$.

\textbf{Dynamic regret.} During the interval $[0, T]$, the agent operates according to a sequence of policies ${\pi_1, \pi_2,\dots,\pi_T}$. Drawing from the learning procedure outlined previously, we define the time-varying dynamic regret $\mathfrak{R}(T) := \sum_{t=1}^{T} \left(V^{*}_{t} - V^{\pi_t}_{t} \right)$, where $V^{*}_{t}$ represents the optimal policy value at time $t$ and $V^{\pi_t}_{t}$ is the value function obtained by executing policy $\pi_t$ in the MDP $\mathcal{M}_t$.

\textbf{Parallel process of policy learning and data collection. }
In our formalization of \hl{policy learning in a non-stationary} environment, the policy learning phase and the data collection phase (interaction) occur concurrently. In this context, the number of trajectories an agent can execute between the unit times $(t, t+1), \forall t \in [T-1]$, typically depends on the system's control frequency or its hardware capabilities. However, for the purpose of our analysis, we assume that the agent executes one trajectory per unit time. This means that at time $t$, the agent has rolled out a total of $t$ trajectories. 

Before the first episode, the agent determines several key parameters:

\begin{enumerate}
    \item \textbf{Frequency of Policy Updates:} The agent decides on the number of updates, denoted as \( M \in \mathbb{N} \) times.
    \item \textbf{Timing of Policy Updates:} The update times are set as a sequence \( \{ t_1, t_2, t_3, \ldots, t_M \} \) within \( [0, T] \).
    \item \textbf{Extent of Each Update:} The policy update iteration sequence is defined as \( \{ G_1, G_2, \ldots, G_M \} \).
\end{enumerate}

Specifically, at each time \( t_m \in [0, T] \) where $m\in[M]$, the agent updates its policy for \( G_m \in \mathbb{N} \cup \{0\} \) iterations, using all previously collected trajectories. We assume that each policy iteration corresponds to one second in real-time. The policy then remains fixed for \( N_m \in \mathbb{N} \cup \{0 \} \) seconds after the updates, where it is determined as \( N_m = t_{m+1} - (t_m + G_m) \). The next episode starts immediately at time \( t_{m+1} = t_{m} + G_m + N_m \). Without loss of generality, we assume that $t_1=0$, and therefore $t_m = \sum_{i=1}^{m-1} (N_i + G_i)$ holds. Also, we define the $m^{\text{th}}$ policy update interval as \( \mathcal{G}_m := [t_m, t_m + G_m) \) and the $m^{\text{th}}$ policy hold interval as \( \mathcal{N}_m := [t_m + G_m, t_{m+1}) \). For notational simplicity, we denote $\bar{B}_r (t_m, t_m + G_m)$, $\bar{B}_r (t_m + G_m, t_{m+1})$, $\bar{B}_p (t_m, t_m + G_m)$ and $\bar{B}_p (t_m + G_m, t_{m+1})$ as $\bar{B}_r (\mathcal{G}_m)$, $\bar{B}_r (\mathcal{N}_m)$, $\bar{B}_p (\mathcal{G}_m)$ and $\bar{B}_p (\mathcal{N}_m)$, respectively. 

\textbf{How to determine $\boldsymbol{\{\pi_1,\pi_2,...,\pi_T\}}$. } At time $t_m$, the agent executes the policy $\pi_{t_m}$ and starts optimizing the policy for $G_m$ seconds.  During this optimization, after \( g \) iterations (seconds), where \( g \in [G_m] \), the agent executes the most recently updated policy \( \pi^g_{t_m} \). This updated policy represents the \( g^{th} \) iteration of optimization from the initial policy \( \pi_{t_m} \). Therefore, during the policy update interval \( \mathcal{G}_m \), specifically at time \( t_m + g \), the policy \( \pi_{t_m+g} \) is equivalent to \( \pi^g_{t_m} \). Subsequently, throughout the policy hold interval \( \mathcal{N}_m \), the agent continues to execute the latest updated policy, denoted as \( \pi_t = \pi^{G_m}_{t_m} \) for every \( t \) within \( \mathcal{N}_m \).
\begin{figure}[ht]
    \centering
    \includegraphics[width=0.95\columnwidth]{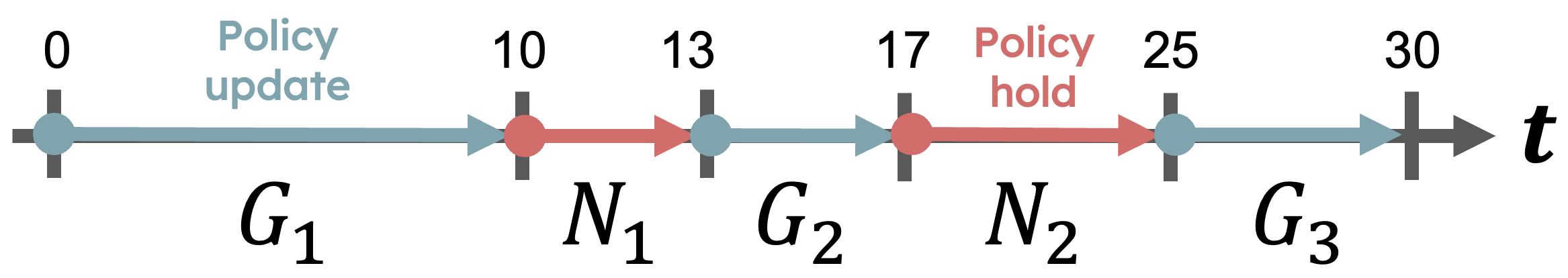}
    \caption{Parallel process of policy learning and data collection.}
    \label{fig:problemstatement}
\end{figure}

\textbf{Example.} Figure \ref{fig:problemstatement} illustrates our problem setting. For a given time duration between \( t = 0 \) and \( t = 30 \), suppose that the agent has chosen the frequency of policy updates as \( M = 3 \) and the update time sequence as \( t_1 = 0, t_2 = 13, t_3 = 25 \), along with the policy update durations \( G_1 = 10, G_2 = 4, G_3 = 5 \). The agent begins the first episode at \( t = 0 \) with a random policy \( \pi_0 \). Subsequently, during times \( t = 1, 2, \ldots, 10 \), the agent continuously executes updating policies \( \pi_0^1, \pi_0^2, \ldots, \pi_0^{10} \), respectively, and then employs the latest updated policy \( \pi_0^{10} \) at times \( t = 11, 12, 13 \). Following this, the agent operates with policies \( \pi_{13}^1, \pi_{13}^2, \ldots, \pi_{13}^{4} \) during the period \( t = 14, 15, 16, 17 \), where \( \pi_{13} = \pi_0^{10} \). Lastly, it executes with the most recently updated policy \( \pi_{13}^{4} \) during the time \( t = 18, \ldots, 25 \).
\section{Method}
\label{Method}

To implement a real-time inference mechanism, particularly emphasizing the prediction-based control approach of \emph{``predicting the future in the past,''} we introduce a model-free proactive algorithm, detailed in Algorithm \ref{algo1}. This approach is based on the proactive evaluation of policies. At policy update time $t_m$, our proposed algorithm forecasts the future $Q$ value of time $t_{m+1}$ based on previous trajectories and then optimizes the future policy for duration $\mathcal{G}_m$ based on foreacasted future $Q$ value. For all $t \in [0,T]$, we denote the estimated value of $Q$ based on the past trajectories as $\widehat{Q}_t$ and the optimal value of $Q$ as $Q^*_t$. We also denote the future $Q$ value of time $t_{m+1}$  which was forcasted at time $t_m$ as $\widetilde{Q}_{t_{m+1} | t_m}$. During the time duration $\mathcal{G}_m$, we determine the policies $\{\pi^{g}_{t_m}\}_{g=1}^{G_m}$ by utilizing the Natural Policy Gradient \cite{kakade2001natural} with the entropy regularization method based on $\widetilde{Q}_{t_{m+1} | t_m}$ as follows:
\begin{align*}
    \pi^{g+1}_{t_{m}} (\cdot | s) &\propto \left( \pi^{g}_{t_{m}}(\cdot | s) \right)^{1 - \frac{\eta \tau}{ 1- \gamma}} \exp{\left( \frac{\eta \tilde{Q}_{t_{m+1|t_m}}}{1-\gamma} \right)} \\ 
    & \text{s.t.}\quad || \tilde{Q}_{t_{m+1}|t_m} - Q^*_{t_{m+1}} ||_{\infty} = \delta_m^f
\end{align*}
where $\eta$ is a learning rate, $\tau$ is an entropy regularization parameter and  $\delta_m^f$ is the maximum forecasting error at time step $t_m$.

There are various methods to forecast $Q_{t_{m+1}|t_m}$ based on past $Q$ estimates $\{\widehat{Q}_{t}\}^{t_m}_{t=0}$. In this work, we provide analytical explanations on how the forecasting error can be bounded by the past $l$ uncertainties ($Q$ estimation errors) and the intrinsic uncertainty of the future environment (local variation budgets). For any $t$, we refer to $\epsilon_{t}$ as the maximum $Q$ estimation error if $||\widehat{Q}_{t} - Q^{*}_{t} ||_{\infty} \leq \epsilon_t$ holds. To simplify the presentation, we drop the term ``maximum'' when it is clear from the context. 
\begin{proposition}[Linear forecasting method with bounded \( l_2 \) norm]
    Consider a past reference length \( l_p \in \mathbb{N} \) and define \( \boldsymbol{w} := [w_{t_m-l_p+1}, \ldots, w_{t_m-1}, w_{t_m}]^\top \). We forecast \( \tilde{Q}_{t_{m+1}|t_m} \) as a linear combination of the past \( l_p \)-estimated \( Q \) values, namely \( \tilde{Q}_{t_{m+1}|t_m} = \sum_{t=t_m-l_p+1}^{t_m} w_t \hat{Q}_t \), where the condition \( \lVert \boldsymbol{w} \rVert_2 \leq L \) holds for some $L$. Then, \( \delta_m \) can be bounded by 
    \begin{align*}
            \delta_m^f &\leq L \sqrt{ \sum_{t=t_m-l_p+1}^{t_m} 2 \left(\max(u_t,\epsilon_t)\right)^2 }+ l_p(L+1)  \\ 
            \quad & \left( \frac{1-\gamma^H}{1-\gamma} r_{\text{max}} \right)
    \end{align*}
    where $u_t :=  \frac{1-\gamma^H}{1-\gamma} \left( B_r(t,t_{m+1}) + \frac{r_{max}}{1-\gamma} B_p(t,t_{m+1})\right)$ and $r_{\text{max}} := \max_{t,s,a} |R_t(s,a)|$.
    \label{prop1}
\end{proposition}
Proposition \ref{prop1} shows that utilizing a low-complexity forecasting model provides that the $m^\text{th}$ maximum forecasting error is bounded by intrinsic environment uncertainty of future  $\{u_t\}_{t=t_m-l_p-1}^{t_m}$ and past uncertainties $\{\epsilon_t\}_{t=t_m-l_p-1}^{t_m}$ due to finite samples.

Compared to previous studies on finite-time $Q$ value convergence with asynchronous updates \cite{pmlr-v125-qu20a,Even-Dar2004Qlearning}, our work primarily focuses on how strategic policy update intervals affect an upper bound on the dynamic regret, leaving room for future exploration of $Q$ convergence rate improvement. This will be discussed in more detail in Section \ref{Theoretical Analysis}. 

In the remainder of this section, we investigate in Proposition \ref{prop2} and Corollary \ref{cor:maxforecastingerrorbound} how an $\epsilon_t$-accurate estimate of past $Q$ value establishes a lower bound condition on $\{N_i\}^{m-1}_{i=1}$ and $\{G_i\}_{i=1}^{m-1}$.
\begin{proposition}[Past uncertainty with sample complexity \cite{pmlr-v125-qu20a}]
    For any $\kappa >0$ and under some conditions on stepsizes, if ${t} \geq \frac{(|S||A|)^{3.3}}{(1-\gamma)^{5.2} \epsilon_t^{2.6}}$, then $||\widehat{Q}_{t} - Q^{*}_{t} ||_{\infty} \leq \epsilon_t$ holds.
    \label{prop2}
\end{proposition}
Proposition \ref{prop2} highlights that the lower bound conditions of $\{N_i\}^{m-1}_{i=1}$ and $\{G_i\}_{i=1}^{m-1}$ are useful to reach $\epsilon_t$-accurate estimate of $Q$ value for asynchronous $Q$-learning method on a single trajectory. The upper bound of $\delta_m^f$ could be better minimized by taking $\max(u_t,\epsilon_t) = u_t$ for all $t \in [t_m-l+1,t_m]$. This requires $t \geq \frac{(|S||A|)^{3.3}}{(1-\gamma)^{5.2} u_t^{2.6}}$ to hold for all $t \in [t_m-l+1,t_m]$. Note that $t_m=\sum^{m-1}_{i=1} (N_i + G_i)$ holds. Therefore, for $j=1,2,...,l_p$, we have $\sum^{m-1}_{i=1} (N_i + G_i) - j +1 \geq \frac{(|S||A|)^{3.3}}{(1-\gamma)^{5.2} u_{t_m-j+1}^{2.6}}$. Then, the upper bound can be simplified without past uncertainty terms as follows. 
\begin{corollary}[Maximum forecasting error bound]
    For $j=1,2,...,l_p$, if $\{N_i\}_{i=1}^{m-1}$ and $\{ G_i\}_{i=1}^{m-1}$ satisfy the condition $\sum^{m-1}_{i=1} (N_i + G_i) - j +1 \geq \frac{(|S||A|)^{3.3}}{(1-\gamma)^{5.2} u_{t_m-j+1}^{2.6}}$, then $\delta_f$ is bounded by 
    \begin{align*}
        \delta_f &\leq L u_{\text{max}} \sqrt{2l_p} + l_p(L+1) \left( \frac{1-\gamma^H}{1-\gamma} r_{\text{max}} \right)
    \end{align*}
     where $\delta_f := \max_{m \in [M]}  \delta_m^f$ is a maximum forecasting error and $u_{\text{max}} := \max_{m \in [M]}  u_{t_m-l_p+1}$.
    \label{cor:maxforecastingerrorbound}
\end{corollary}
Corollary \ref{cor:maxforecastingerrorbound} shows how the forecasting error $\delta_f$ is bounded with future environment's uncertainty $u_{max}$ with lower bound conditions on $\{N_i\}^{m-1}_{i=1}$ and $\{G_i\}_{i=1}^{m-1}$. By collecting more trajectories per the unit time $(t, t+1)$, we can significantly relax the lower bound condition, going beyond our initial assumption (see Section \ref{problem_Statement}).
\begin{algorithm}[tb]
   \caption{Forecasting Online Reinforcement Learning}
    \begin{algorithmic}[1]
       \STATE {\bfseries Input:} Total time $T$, Policy update duration sets $\{H_1,..,H_M\}$, $\{G\}_{1:K}$, Dataset $\mathcal{D}$
        \STATE {\bfseries Init:} $m=0$, $\pi_1=$ random policy
            \FOR {$t=\{1,2,...,T\}$}
                \STATE Rollout $H$ steps trajectory with policy $\pi_t$ and save a trajectory to $\mathcal{D}$
                \IF{$t \in \{ t_1,t_2,...,t_M\}$}
                    \STATE $m \leftarrow m+1$
                    \STATE $\widetilde{Q}_{t_{m+1}|t_m} = \texttt{ForQ} (\mathcal{D})$ \textcolor{blue}{\footnotesize /* Forecast future Q */}
                \ENDIF
                \IF{$t \in [t_{m} + G_m)$}
                    \STATE $\pi_{t+1} = \texttt{Update}(\pi_t, \eta, \tau, \gamma,\widetilde{Q}_{t_{m+1}|t_m})$ \textcolor{blue}{\footnotesize /* Update Policy */}
                \ELSIF{$t \in [t_{m}+G_m + N_m)$}
                    \STATE $\pi_{t+1} = \pi_t$ \textcolor{blue}{\footnotesize /* Pause policy update*/}
                \ENDIF
            \ENDFOR
    \end{algorithmic}
    \label{algo1}
\end{algorithm}
\section{Theoretical Analysis}
\label{Theoretical Analysis}
In this section, we provide a dynamic regret analysis to investigate how policy hold durations \( \{N_1, N_2, \ldots, N_M\} \) influence the minimization of dynamic regret. We initially decompose the regret into two main components and calculate upper bounds on these components in Subsection \ref{Regret analysis}. Subsequently, in Subsection \ref{Theoretical insight}, we further divide the overall upper bound of regret into three distinct terms and investigate how $N_m$ modulates each of these terms, except for the future forecasting regret term. Finally, in Subsection \ref{Numerical verification on theoretical insights}, we present numerical experiments that demonstrate variations in the regret upper bound in response to different  $N_m$ values under different aleatoric uncertainties.

\subsection{Regret analysis}
\label{Regret analysis}
We define the dynamic regret between times \( t_m \) and \( t_{m+1} \) as \( \mathfrak{R}_m(T) \), which is given by
$\mathfrak{R}_m(T) := \sum_{t=t_m}^{t_{m+1}} \left( V^*_t - V^{\pi_t}_t \right).$
The \( m^{\text{th}} \) dynamic regret, \( \mathfrak{R}_m(T) \), can be decomposed into two components, named Policy update regret and Policy hold regret, as follows:
\begin{align*}
    \mathfrak{R}(T)
    &=\sum_{m=1}^{M} \bigg( \underbrace{\sum_{t \in \mathcal{G}_m} \left( V_t^* - V_t^{\pi_t} \right)}_{\text{Policy update regret}} + \underbrace{\sum_{t \in \mathcal{N}_m} \left( V_t^* - V_t^{\pi_t} \right)}_{\text{Policy hold regret}} \bigg).
\end{align*}
The policy update regret and the policy hold regret will be studied next.
\begin{lemma}[Policy update regret]
    Let $\bar{B}(\mathcal{G}_m) := C_4 \bar{B}_r(\mathcal{G}_m)+C_5 \bar{B}_p(\mathcal{G}_m)$. For all $t \in \mathcal{G}_m$ where $m \in [M]$, it holds that 
    \begin{align*}
        \sum_{t \in \mathcal{G}_m} &\left( V_t^* - V_t^{\pi_t} \right) \leq \frac{C_1}{\eta \tau} \cdot \Big( 1- (1-\eta \tau)^{G_m} \Big) \\
        &+ G_m \Big( C_2 \delta^f_m + C_3\Big) + \bar{B}(\mathcal{G}_m)
    \end{align*}
     where 
     $C_1= (\gamma+2)\big(|| Q^*_{t_m} - Q_{t_m}||_\infty + 2\tau (1- \frac{\eta \tau}{1- \gamma} $ $|| \log \pi_{t_m}^* - \log \pi_{t_m}||_\infty) \big),C_2 = \frac{2(\gamma+2)}{1-\gamma}\left( 1 + \frac{\gamma}{\eta \tau} \right), C_3 = \frac{2 \tau \log |\mathcal{A}|}{1-\gamma}, C_4= \frac{2(1-\gamma^H)}{1-\gamma},C_5=\frac{\gamma}{1-\gamma} \cdot \left( \frac{1-\gamma^H}{1-\gamma} - \gamma^{H-1} H \right) + \frac{1-\gamma^H}{1-\gamma} \cdot \frac{r_{max}}{1-\gamma}$.
    \label{lemma1}
\end{lemma}

\begin{lemma}[Policy hold regret]
    Let $\bar{B}(\mathcal{N}_m) := C_4 \bar{B}_r(\mathcal{N}_m)+C_5 \bar{B}_p(\mathcal{N}_m)$. For all $t \in \mathcal{N}_m$ where $m \in [M]$, it holds that
    \begin{align*}
        \sum_{t \in \mathcal{N}_m}  &\left( V_t^* - V_t^{\pi_t} \right) \leq N_m \cdot \Big( C_1(1-\eta \tau)^{G_m} \\
        &+ C_2 \delta_m^f + C_3 \Big) + \bar{B}(\mathcal{N}_m)
    \end{align*}
    \label{lemma2}
\end{lemma}
where $C_1,C_2,C_3,C_4,C_5$ are the constants defined in Lemma \ref{lemma1}.

By leveraging Lemmas \ref{lemma1} and \ref{lemma2}, the dynamic regret $\mathfrak{R}(T)$ will be bounded below.  
\begin{theorem}[Dynamic regret]
Let $\bar{B}(t_m,t_{m+1}) := \bar{B}(\mathcal{N}_m) + \bar{B}(\mathcal{G}_m)$. Then, it holds that 
\begin{align*}
    &\mathfrak{R}(T) \leq \sum_{m=1}^{M} \bigg( \underbrace{\frac{C_1}{\eta \tau} + \left( N_m C_1 - \frac{C_1}{\eta \tau} \right) \left( 1- \eta \tau\right)^{G_m}}_{\text{policy optimization regret} (\mathfrak{R}^\pi_m)} \\
    &+ \underbrace{(N_m +G_m)(C_2 \delta_m^f +C_3)}_{\text{Q function forecasting regret}(\mathfrak{R}^{\text{f}}_m)} + \underbrace{\bar{B}(t_m,t_{m+1})}_{\text{non-stationarity regret}(\mathfrak{R}^{\text{env}}_m)} \bigg).
\end{align*}
    \label{theorem1}
\end{theorem}
In Theorem \ref{theorem1}, we articulate the decomposition of $\mathfrak{R}_m(T)$ into three terms: the policy optimization regret, denoted as $\mathfrak{R}^\pi_m$, $Q$ value forecasting regret, denoted as $\mathfrak{R}^f_m$, and non-stationarity regret, denoted by $\mathfrak{R}^{\text{env}}_m$. Now, by extending the upper bound of the forecasting error regret to \(\sum_{m=1}^M \mathfrak{R}^f_m \leq \sum_{m=1}^{M} (N_m + G_m)(C_2 \delta_f + C_3) = T(C_2 \delta_f + C_3) \leq T \left( C_2 \left(L u_{\text{max}} \sqrt{2l_p} + l_p(L+1) \left( \frac{1-\gamma^H}{1-\gamma} r_{\text{max}} \right) \right) + C_3 \right)\), we find that its upper bound is independent from $\{N_i, G_i\}_{i=1}^{m}$ and satisfies a sublinear convergence rate to the total time $T$ for any $l_p = (1/T)^\alpha , \alpha >0$. 

Expanding on the independence of $\{ N_i,G_i\}_{i=1}^m$ from the upper bound of $\sum_{m=1}^{M} \mathfrak{R}^f_m$, we will show how $N_m$ balances between $\mathfrak{R}^\pi_m$ and $\mathfrak{R}^{env}_m$, followed by minimizing the upper bound of $\mathfrak{R}_m(T)$ in the next subsection.

\subsection{Theoretical insight}
\label{Theoretical insight}

One crucial theoretical insight to be deduced from Theorem \ref{theorem1} is what nonzero value of $N_m$ strikes a balance between $\mathfrak{R}^\pi_m$ and $\mathfrak{R}^{env}_m$. Our insights begin with the analysis of $\mathfrak{R}_m(T)$. We start by considering a fixed time interval \( [t_m, t_{m+1}] \), which brings up the constraint \( N_m + G_m = t_{m+1} - t_{m} \). The initial aspect of our investigation addresses whether a nonzero value of \( N_m \) offers any advantage in a \emph{stationary} environment.

\begin{lemma}[Optimal \( N^*_m, G^*_m \) for \( \mathfrak{R}^\pi_m \)]
    Given a fixed time interval \( [t_m, t_{m+1}] \), the optimal values \( N^*_m \) and \( G^*_m \) that minimize \( \mathfrak{R}^\pi_m \) are determined as \( N^*_m = 0 \) and \( G^*_m = t_{m+1} - t_m \), respectively.
    \label{lemma:R_pi}
\end{lemma}
Since $\mathfrak{R}^{env}_m=0$ is satisfied in stationary environments (see Definition \ref{def:stationary}), Corollary \ref{cor:optimalstatinoaryenv} ensues from Lemma \ref{lemma:R_pi}.
\begin{corollary}[Optimal \( N^*_m, G^*_m \) in Stationary Environments]
    Consider a stationary environment. The upper bound of \( \mathfrak{R}_m \) 
    achieves its minimum when
    \( N_m = 0 \) and \( G_m = t_{m+1} - t_m \).
    \label{cor:optimalstatinoaryenv}
\end{corollary}

What Corollary \ref{cor:optimalstatinoaryenv} states is intuitively straightforward. This is because in scenarios where the time sequence of the policy update $(t_1, t_2, \ldots, t_m)$ is fixed, maximizing the policy update duration is advantageous without considering forecasting errors. However, we claim that \( N_m \) plays an important role in a non-stationary environment, i.e., positive $N^*_m$ minimizes the upper bound of $\mathfrak{R}_m(T)$. We first develop the following proposition. 

\begin{proposition}[Existence of Positive $N^*_m$ for $\mathfrak{R}^{\text{env}}_m$]
    In a non-stationary environment, consider any given time interval \( [t_m, t_{m+1}] \) satisfying \( t_{m+1} - t_{m} \geq 2 \). Under these conditions, there exists a number \( N_m \) within the open interval \( (0, t_{m+1} - t_m) \) that minimizes \( \bar{B}(t_m, t_{m+1}) \).
    \label{proposition:R_env}
\end{proposition}
\begin{figure}
    \subfigure[]{\includegraphics[width=0.23\textwidth]{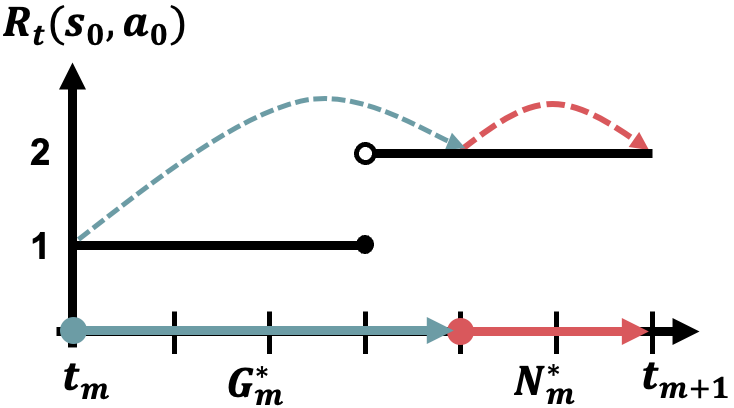}}
    \subfigure[]{\includegraphics[width=0.23\textwidth]{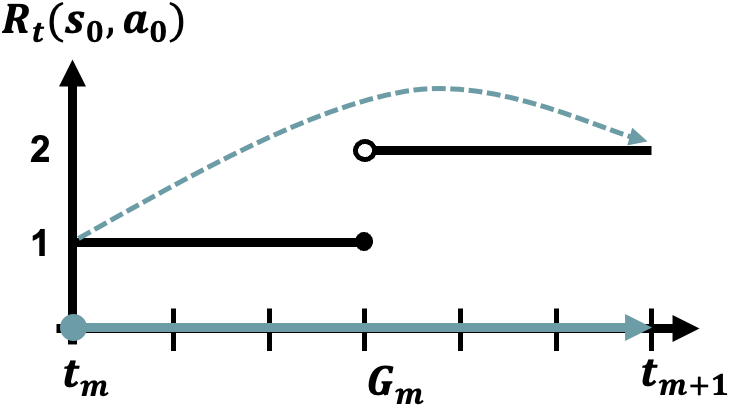}}
    \caption{Optimal solutions for $\min_{G_m,N_m} \bar{B}(t_m,t_{m+1})$ are $(G_m^*,N_m^*) = (4,2), (2,4)$. (a) $G_m=4,N_m=2$. (b) $G_m=6,N_m=0$ }
    \label{fig:GNexample}
\end{figure}
One way to intuitively understand Proposition \ref{proposition:R_env} is exemplified in Figure \ref{fig:GNexample}. Consider a non-stationary environment where the reward abruptly changes only at state $s_0$ and action $a_0$. Suppose that $C_4=C_5=1$. Then $\min \bar{B} (t_m,t_{m+1}) =1$ and its solution is attainable at $(G^*_m,N^*_m) = (4,2)$ and $(2,4)$ (Figure \ref{fig:GNexample} (a)), while in the case where $G_m=6,N_m=0$ yields $\bar{B} (t_m,t_{m+1}) =3$ (Figure \ref{fig:GNexample} (b)). Both subfigures optimize the policy toward the forecasted future $Q$ value of time $t_{m+1}$, but the time that the agent stops to update the policy ($t=t_m +G_m$) determines how much the agent would be conservative with respect to the future reward prediction.

Based on Proposition \ref{proposition:R_env}, we introduce the surrogate optimal solution \( (G_m^*, N_m^*) \) for the non-stationarity regret \( \mathfrak{R}^{env}_{m} \). According to Corollary \ref{cor:Exponential order cumulative variation budget}, it holds that \( \bar{B}_r(\mathcal{N}_m) \) is bounded by \( \sum_{t=t_m+G_m}^{t=t_m+G_m+N_m-1} \alpha_r^{t-(t_m+G_m)}B_r^{\text{max}}(\mathcal{N}_m) \), and similarly, \( \bar{B}_r(\mathcal{G}_m) \) is bounded by \( \sum_{t=t_m}^{t=t_m+G_m-1} \alpha_r^{t-t_m}B_r^{\text{max}}(\mathcal{G}_m) \). For brevity, we use the notation \( \alpha_{\diamond,1} = \alpha_\diamond (\mathcal{G}_m) \) and \( \alpha_{\diamond,2} = \alpha_\diamond (\mathcal{N}_m) \), and similarly for \( B^{\text{max}}_{\diamond,1}  = B_\diamond^{\text{max}}(\mathcal{G}_m)\) and \( B^{\text{max}}_{\diamond,2}=B^{\text{max}}_{\diamond}(\mathcal{N}_m) \), where \( \diamond \) is either \( r \) or \( p \). Furthermore, we define \( \alpha_\square \) as the \( \max \left( \alpha_{r,\square} , \alpha_{p,\square}\right) \), and \( B^{\text{max}}_\square \) as the \( \max \left( B^{\text{max}}_{r,\square} , B^{\text{max}}_{p,\square} \right) \), where \( \square \) is either \( 1 \) or \( 2 \).

\begin{lemma}[Surrogate optimal $(G^*_m,N^*_m)$ for $\mathfrak{R}^{env}_{m}$]
    For given $m,t_m,t_{m+1}$,  the surrogate optimal policy update and policy hold variables that minimize the upper bound of $\mathfrak{R}^{env}_m$ are 
    $$N^*_m = \argmin_{N_m \in \{ \lfloor \tilde{N}^*_m \rfloor, \lfloor \tilde{N}^*_m \rfloor +1 \}} \mathfrak{R}^{env}_{m}(N_m,G_m)$$
    and $$G^*_m = t_{m+1}-t_m - N^*_m,$$
    where
    $$\tilde{N}^*_m = \frac{1}{\ln \left( \alpha_2 / \alpha_1 \right)} \cdot \ln \left( \frac{\ln \alpha_1 / (\alpha_1-1) }{\ln \alpha_2 / (\alpha_2-1) } \cdot \alpha_1^{t_{m+1}-t_m}\cdot \frac{B^{\text{max}}_1}{B^{\text{max}}_2} \right).$$
    \label{lemma:R_env}  
\end{lemma}

Note that Lemma \ref{lemma:R_env} provides a nonzero suboptimal $N^*_m$ that minimizes the non-stationary regret $\mathfrak{R}^{env}_{m}$. Now, we combine Lemmas \ref{lemma:R_pi} and \ref{lemma:R_env} to find the suboptimal $N_m^*$ and $G^*_m$ that minimize the upper bound of $\mathfrak{R}_m(T)$.
\begin{figure}
    \centering
    \subfigure[]{\includegraphics[width=0.23\textwidth]{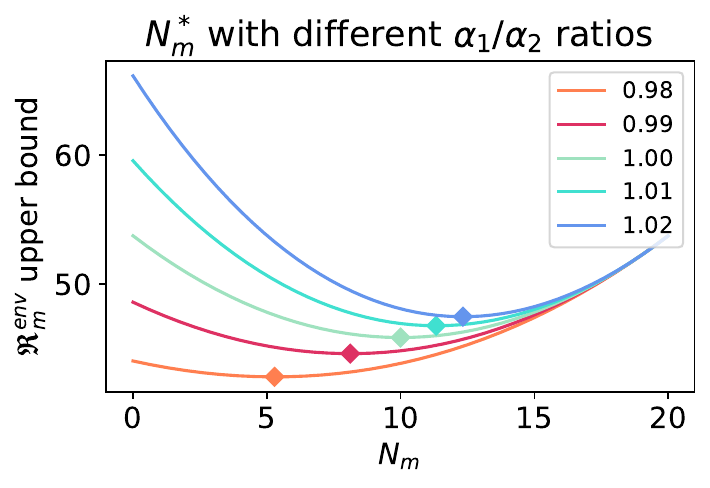}}
    \subfigure[]{\includegraphics[width=0.23\textwidth]{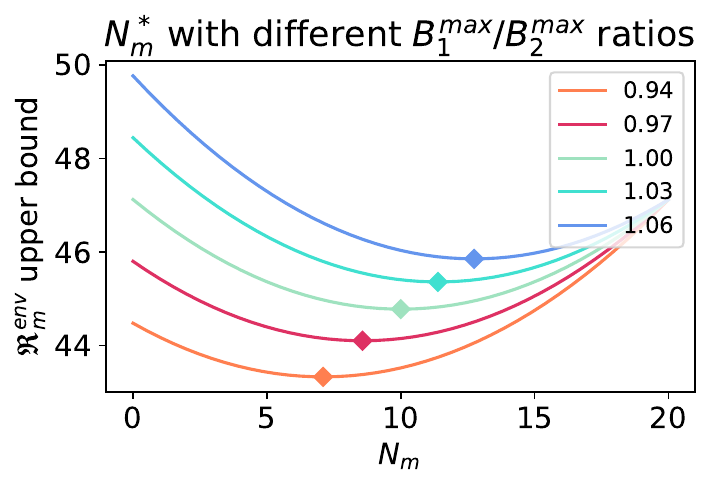}}
    \caption{$\mathfrak{R}^{\text{env}}_m$ upper bound with different environmental hyperparameters. $\blacklozenge$ denotes the minimum of each function graph. (a) $\alpha_1 / \alpha_2 \in \{ 0.98, 0.99, 1.0, 1.01, 1.02\}$. (b) $B_1^{\text{max}} / B_2^{\text{max}} \in \{ 0.94, 0.97, 1.0, 1.03, 1.06\}$.}
    \label{fig:aBchange}
\end{figure}
\begin{figure}
    \centering
    \subfigure[]{\includegraphics[width=0.23\textwidth]{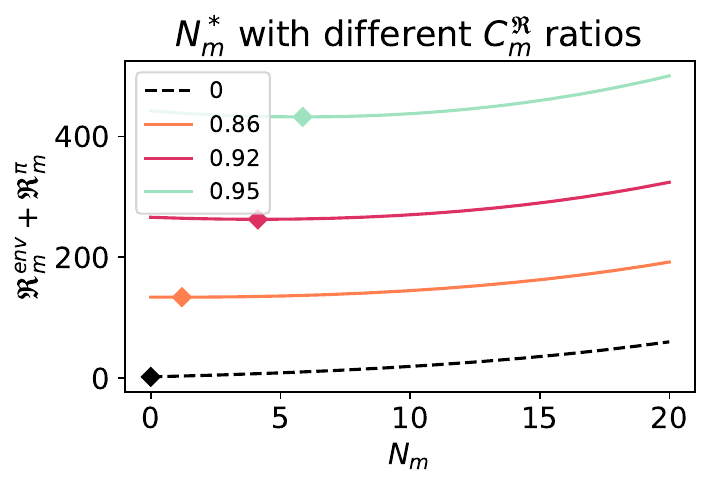}}
    \subfigure[]{\includegraphics[width=0.23\textwidth]{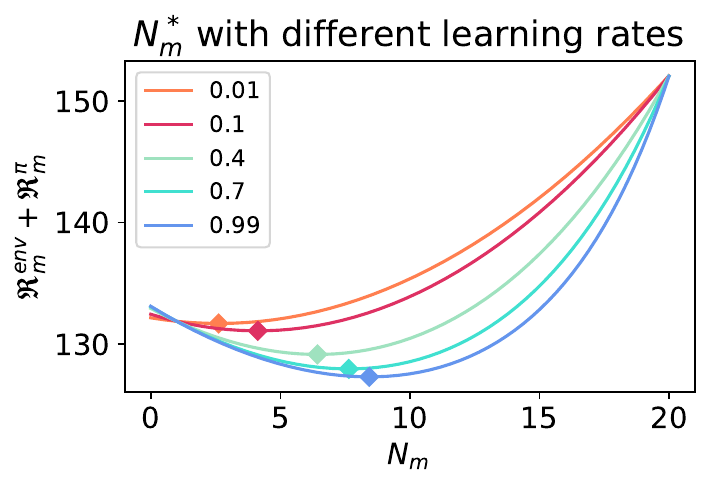}}
    \caption{$\mathfrak{R}^{\text{env}}_{m}+\mathfrak{R}^{\pi}_{m}$ upper bound with different $C_{\mathfrak{R}}$ ratios and learning rates. $\blacklozenge$ denotes the minimum of each function graph. (a) $C^{\mathfrak{R}}_m \in \{ 0, 0.86, 0.92, 0.95 \}$. (b) Learning rates $\in \{ 0.01, 0.1, 0.3, 0.7, 0.99\}$.}
    \label{fig:clrchange}
\end{figure}
\begin{theorem}[Surrogate optimal \( (G^*_m, N^*_m) \) for \( \mathfrak{R}_m \)]
    For given $m,t_m,t_{m+1}$, the surrogate optimal policy update variable $N_m$ and surrogate policy hold variable $G_m$ that minimize the upper bound of $\mathfrak{R}_m$ satisfy the following equation:
    \begin{align*}
        & C_1 \left( (N_m - 1) \ln (1 - \eta \tau) - 1 \right) (1 - \eta \tau)^{G_m} + \\
        & (C_4 + C_5) \left( \frac{\ln \alpha_1}{\alpha_1 - 1} B^{\text{max}}_1 \alpha_1^{G_m} - \frac{\ln \alpha_2}{\alpha_2 - 1} B^{\text{max}}_2 \alpha_2^{N_m} \right) = 0,
    \end{align*}
    where \( C_1, C_4, C_5, \eta, \tau, \alpha_1, \alpha_2, B^{\text{max}}_1, \) and \( B^{\text{max}}_2 \) are constants or parameters specific to the system under consideration.
    \label{theorem_optimalGN}
\end{theorem}
Apart from Lemma \ref{lemma:R_env}, Theorem \ref{theorem_optimalGN} does not provide a closed-form solution. Consequently, we will conduct some numerical experiments to understand how $N^*_m$ and $G^*_m$ change to the hyperparameters of Theorem \ref{theorem_optimalGN} in the next subsection. 

\subsection{Numerical analysis of theoretical insights}
\label{Numerical verification on theoretical insights}
Figures \ref{fig:aBchange} and \ref{fig:clrchange} show how the surrogate optimal $N^*_m$ changes with different parameter choices. Figure \ref{fig:aBchange} shows how $N^*_m$ changes with different parameters of the environment intrinsic uncertainty. Note that $(\alpha_1, B^{\text{max}}_1)$ and $(\alpha_2, B^{\text{max}}_2)$ represent the magnitude (severity) of the intrinsic uncertainty of the environment during the policy update phase ($\mathcal{G}_m$) and the policy hold phase ($\mathcal{N}_m$), respectively. The two subfigures of Figure \ref{fig:aBchange} not only support the importance of holding $N^*_m$, but also show the necessity of keeping the policy hold phase longer if the uncertainty of the environment during the policy update phase ($\alpha_1,B^{\text{max}}_1$) is greater than that of the policy hold phase ($\alpha_2, B^{\text{max}}_2$). Moreover, Figure $\ref{fig:clrchange}$ (a) shows that increasing $N^*_m$ provides a better performance if the environment regret term dominates the regret \( \mathfrak{R}^{\text{env}}_{m} + \mathfrak{R}^\pi_m \). We define the dominant ratio $C^\mathfrak{R}_m$ as $C^\mathfrak{R}_m := \int_{t_m}^{t_{m+1}} \mathfrak{R}^{env}_m / (\mathfrak{R}^{env}_m + \mathfrak{R}^{\pi}_m) d\text{t}$. Finally, Figure $\ref{fig:clrchange}$ (b) validates that the surrogate optimal solution is still an acceptable solution and illustrates that the suboptimal gap resulting from relaxing the non-convex upper bound into a convex one is tolerable, as a higher learning rate leads to a fast convergence of $\mathfrak{R}^\pi_m$ and, in turn, intuitively results in a longer $N_m$ within fixed $t_{m},t_{m+1}$. 
\section{Experiments}
\label{Experiments}
In this section, we demonstrate the effectiveness of two key components of the proposed algorithm, forecasting $Q$ value (line 7 of Algorithm \ref{algo1}) and the strategic policy update (line 9 $\sim$ 12 of Algorithm \ref{algo1}). In Subsection \ref{Goal switching cliffworld}, we illustrate how utilizing forecasted $Q$ value yields higher rewards compared to a reactive method in a finite-dimensional environment. Subsequently, in Subsection \ref{Mujoco environment}, we will show how strategically assigning different policy update frequencies provides a higher performance than the continually updating policy method in an infinite-dimensional Mujoco environment, swimmer and halfcheetah. Details of environments and experiments are specified in Appendix \ref{appendix:experiments}.

\subsection{Future $Q$ value estimator}
\label{Future $Q$ value estimator}
For the following experiments in Subsections \ref{Goal switching cliffworld} and \ref{Mujoco environment},  we design the $\texttt{ForQ}$ function as the least-squares estimator \cite{chandak2020optimizing}, namely $\tilde{Q}_{t_{m+1}|t_m}(s,a) = \phi(t_{m+1})^\top w^*(s,a)$ where $\phi:[0,T) \rightarrow \mathbb{R}^{d}$ is a basis function for encoding the time index. For example, an identity basis is $\phi(x) := \{x, 1\}$. Then $w^*(s,a)$ denotes an optimal solution of the least-squares problem for any $s\in \mathcal{S}, a\in \mathcal{A}$, namely $w^*(s,a) = \argmin_{w \in \mathbb{R}^{d \times 1}} || \boldsymbol{Q}(s,a) -  \Phi(X)^\top w||_2$
where $\boldsymbol{Q}(s,a) := [Q_{t_m-l_p+1}(s,a),...,Q_{t_m}(s,a)]^{\top} \in \mathbb{R}^{l_p \times 1}$, $X := [t_m-l_p+1,...,t_m]^{\top} \in \mathbb{R}^{l_p \times 1}$, and $\Phi(X) := [\phi(t_m-l_p+1),...,\phi(t_m)] \in \mathbb{R}^{d \times l_p}$. The solution to the above least-squares problem is $w^*(s,a) = (\Phi(X)^\top \Phi(X))^{-1}\Phi(X)^\top \boldsymbol{Q}(s,a)$.

\begin{figure}
    \centering
    \subfigure[]{\includegraphics[width=0.27\textwidth]{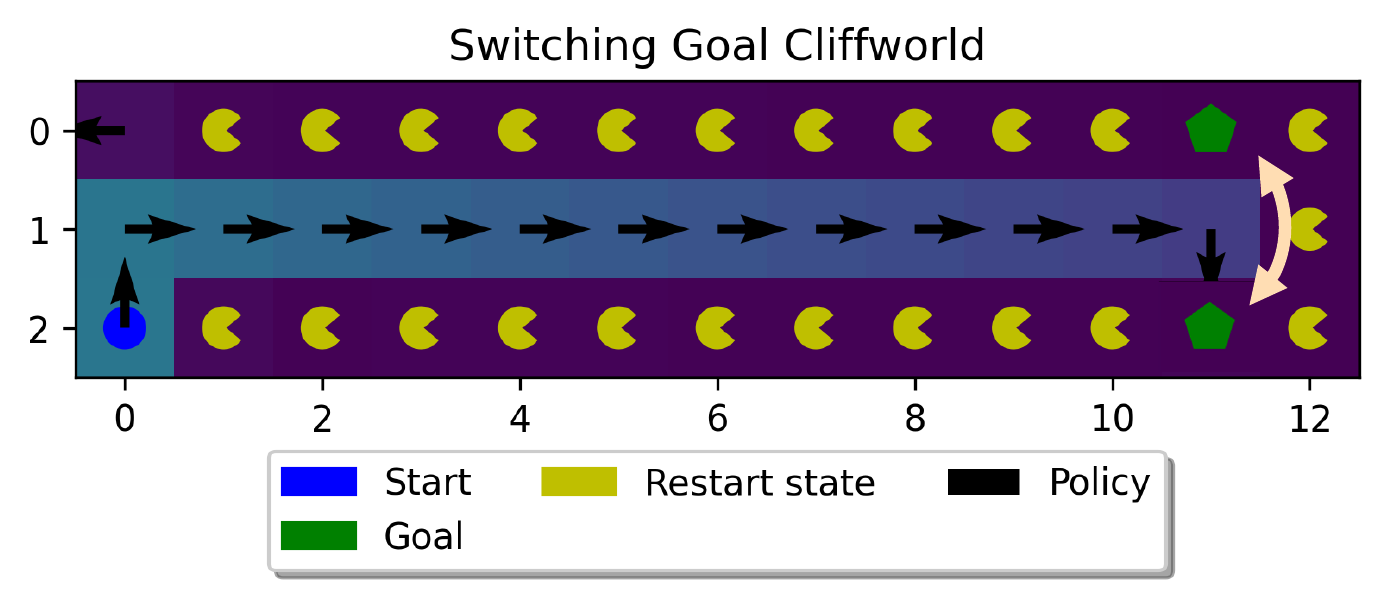}}
    \subfigure[]{\includegraphics[width=0.205\textwidth]{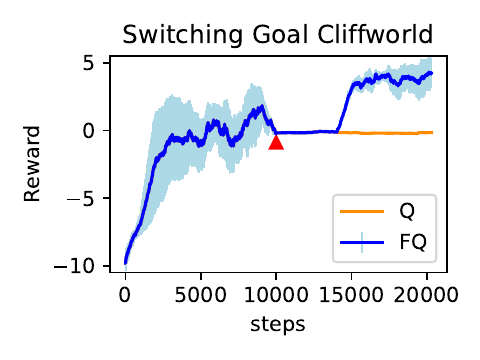}}
    \caption{(a) Swithcing goal cliffworld. (b) Reward per step. A red triangle means the goal point switches at step $=10000$. A shaded area denotes one standard deviation among five different hyperparameter results.}
    \label{fig:experiment_goalswithcingcliff}
\end{figure}

\subsection{Goal switching cliffworld}
\label{Goal switching cliffworld}
We first experiment with a low-dimensional tabular MDP to verify that evaluating the policy by the forecasting method yields a better performance than the reactive method. The environment is the switching goal cliffworld where the agent always starts in the blue circle and a goal switches between two green pentagons (Figure \ref{fig:experiment_goalswithcingcliff} (a)). We use the $Q$-learning algorithm \cite{watkins1992q}, \hl{denoted as \texttt{Q} in Figure \ref{fig:experiment_goalswithcingcliff} (b)}, to evaluate the current policy and compute future policy with future $Q$ estimator, \hl{denoted as \texttt{FQ} in Figure \ref{fig:experiment_goalswithcingcliff} (b)}, proposed in Subsection \ref{Future $Q$ value estimator}. Figure \ref{fig:experiment_goalswithcingcliff} (b) illustrates that after the goal point switches at step $=10000$, the reactive method fails to obtain an optimal policy for the remaining steps. In contrast, the forecasting $Q$ method successfully identifies an optimal policy shortly after step $=15000$.

\begin{figure}
    \centering
    \subfigure[]{\includegraphics[width=0.23\textwidth]{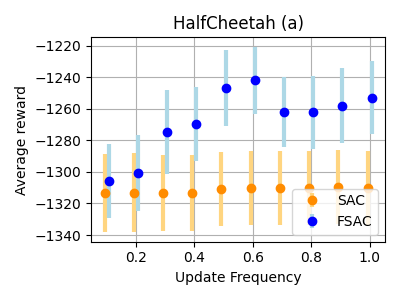}}
    \subfigure[]{\includegraphics[width=0.23\textwidth]{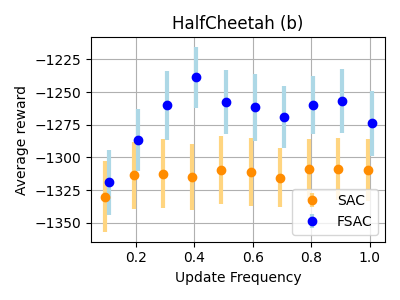}}
    \caption{Halfcheetah environment: blue dots are FSAC and orange dots are SAC. An error bar is 0.5 standard deviation over 36 different hyperparmeter results. (a) Average reward $l_f = 5$. (b) Average reward $l_f = 20$.}
    \label{fig:experiments_hc}
\end{figure}
\begin{figure}
    \centering
    \subfigure[]{\includegraphics[width=0.23\textwidth]{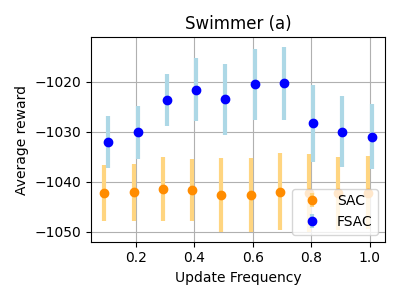}}
    \subfigure[]{\includegraphics[width=0.23\textwidth]{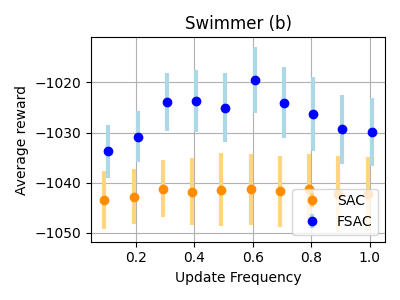}}
    \caption{Swimmer environment: blue dots are FSAC and orange dots are SAC. An error bar is 0.5 standard deviation over 36 different hyperparmeter results. (a) Average reward $l_f = 5$. (b) Average reward $l_f = 15$.}
    \label{fig:experiments_swimmer}
\end{figure}

\subsection{Mujoco environment}
\label{Mujoco environment}
To verify our findings in a large-scale environment, we propose a practical deep learning algorithm, Forecasting Soft-Actor Critic (FSAC), that specifies Algorithm \ref{algo1}. The FSAC algorithm is detailed in Algorithm \ref{alg2} (see Appendix \ref{appendix:algorithms}). Then, we conduct experiments in high-dimensional non-stationary Mujoco environments \cite{todorov2012mujoco}, swimmer, and halfcheetah where the reward changes as the episode goes by \cite{feng2022factored}. We utilize the Soft-Actor Critic (SAC) algorithm \cite{haarnoja2018soft} as a baseline.

In particular, the distinctions between the FSAC and the SAC are the lines $2$, $9 \sim 11$, and $16 \sim 18$ of Algorithm \ref{alg2}. In FSAC, the prediction length $l_f \in \mathbb{N}$ and the update frequency $\gamma_{f} \in (0,1]$ are set as hyperparameters, with $t_m = l_f m$ for all $m \in [M]$ (line $2$). The algorithm forecasts future $Q$ values at every $l_f$ iteration (lines $9 \sim 11$), updating the policy during the interval $(t_m, t_m + \lfloor l_f \gamma_f \rfloor ]$ and keeping it between $(t_{m} + \lfloor l_{f} \gamma_f \rfloor , t_{m+1}]$ (lines $16 \sim 18$). 

Figures \ref{fig:experiments_hc} and \ref{fig:experiments_swimmer} depict the results. In most cases, the FSAC algorithm (indicated by blue dots) yields a higher average return compared to the SAC algorithm (indicated by orange dots). These practical experiments aim to emphasize that  \(\gamma_f = 1.0\) does not necessarily lead to the best average reward. This observation aligns with our theoretical analysis presented in Section \ref{Theoretical insight}, where we demonstrate that a non-negative \(N_m^*\) minimizes the upper bound on dynamic regret. We will elaborate on training and result details in Appendix \ref{appendix:Results}.

\section{Conclusion}

This paper introduces a forecasting online reinforcement learning framework, demonstrating that non-zero policy hold durations improve dynamic regret's upper bound. Empirical results show the forecasting method's advantage over reactive approaches and indicate that continuous policy updates do not always maximize average rewards. For future work, it is crucial to explore methods to minimize the forecasting error to achieve a sharper upper bound. This paper presents work whose goal is implementing real-time control with prediction in environments with unknown uncertainties. A significant societal impact of our research is the narrowing of the gap between simulation-based RL and its real-world applications, along with demonstrating the advantages of pausing \hl{policy} learning in continual learning settings.

\section*{Acknowledgements}
This work was supported by grants from ARO, ONR, AFOSR, NSF, and Noyce Initiative.
\section*{Impact Statement}
This paper presents work whose goal is to advance the field of 
reinforcement learning for real-world application. There are many potential societal consequences 
of our work, none which we feel must be specifically highlighted here.


\bibliography{reference}
\bibliographystyle{icml2024}

\newpage

\appendix
\onecolumn

\section{Related works}
\label{appendix:Related works}

\hl{In this work, we have introduced a forecasting method for non-stationary environments. Before proceeding with our contributions, we first review the existing methods for addressing non-stationary environments in reinforcement learning (RL). Those can be categorized into three main approaches.} 

\hl{One naive approach is to utilize previous RL algorithms that were designed for stationary environments to solve non-stationary environments. Namely, this involves directly applying established RL frameworks for stationary MDPs without additional mechanisms. Usually, this approach involves restarting strategies to handle longer horizon problems in a decision making.} 

\hl{The second approach is model-based methods, which update models to adapt to changing environments. Techniques include using rollout data from the model \cite{janner2019trust,hafner2023mastering}. A few well-known methods include online model updates and identifying latent factors \cite{zintgraf2019varibad,chen2022adaptive,huang2021adarl,feng2022factored,kwon2021rl}. Model-based methods face challenges in non-stationary settings due to difficulties in estimating accurate non-stationary models \cite{cheung2020reinforcement,ding2022non}. To be more specific, \cite{huang2021adarl} explored learning factors of non-stationarity and their representations in heterogeneous domains with varying reward functions and dynamics. \cite{zintgraf2019varibad} proposed a Bayesian policy learning algorithm by conditioning actions on both states and latent tensors that capture the agent's uncertainty in the environment. In a similar manner, \cite{feng2022factored} incorporated insights from the causality literature to model non-stationarity as latent change factors across different environments, learning policies conditioned on these latent factors of causal graphs. Despite these advancements, learning optimal policies conditioned on latent states \cite{zintgraf2019varibad,chen2022adaptive,huang2021adarl,feng2022factored,kwon2021rl} presents significant challenges for theoretical analysis. Recent works \cite{cheung2020reinforcement,ding2022non,ding2022provably} have proposed model-based algorithms with provable guarantees. However, these algorithms are not scalable for complex environments and lack empirical evaluation.}

\hl{The third approach is model-free methods. \cite{al2017continuous} utilized meta-learning among training tasks to find initial hyperparameters of policy networks that can be quickly fine-tuned for new, unseen tasks. However, this method assumes access to a prior distribution of training tasks, which is often unavailable in real-world scenarios. To address this limitation, \cite{finn2019online} proposed the Follow-The-Meta-Leader (FTML) algorithm, which continuously improves parameter initialization for non-stationary input data. Despite its innovation, FTML suffers from a lag in tracking the optimal policy, as it maximizes current performance uniformly over all past samples.}

\hl{To mitigate this lag, \cite{mao2020model} introduced adaptive Q-learning with a restart strategy, establishing a near-optimal dynamic regret bound. \cite{chandak2020optimizing,chandak2020towards} focused on forecasting the non-stationary performance gradient to adapt to time-varying optimal policies. Nevertheless, these approaches are limited by empirical analyses on bandit settings or low-dimensional environments and lack a theoretical performance bound for the adapted policies.  Also, \cite{fei2020dynamic} proposed two model-free policy optimization algorithms based on the restart strategy, demonstrating that their dynamic regret exhibits polynomial space and time complexities. However, these methods \cite{mao2020model,fei2020dynamic} still lack empirical validation and adaptability in complex environments.}

\section{Algorithms}
\label{appendix:algorithms}
\begin{algorithm}
   \caption{\texttt{Update}: Update policy $\pi$ }
   \label{alg:example}
    \begin{algorithmic}[1]
        \STATE {\bfseries Input:} policy $\pi$, learning rate $\eta$, entropy regularization constant $\tau$, discount factor $\gamma$, policy evaluation $\widehat{Q}$ 
        \STATE $Z(s) = \sum_{a \in \mathcal{A}} \left( \pi (a | s) \right)^{1 - \frac{\eta \tau}{ 1- \gamma}} \exp{\left( \eta \widehat{Q} (s,a) / (1-\gamma) \right)}$ 
        \STATE $\pi^\prime (\cdot | s)  = \frac{1}{Z(s)} \cdot \left( \pi (\cdot | s) \right)^{1 - \frac{\eta \tau}{ 1- \gamma}} \exp{\left( \frac{\eta \widehat{Q}(s,\cdot) }{ 1-\gamma} \right)}$
        \STATE {\bfseries Return} $\pi^\prime$
    \end{algorithmic}
     \label{algo:updatePI}
\end{algorithm}
\begin{algorithm}
\caption{Forecasting Soft Actor-Critic}
\begin{algorithmic}[1]
    \STATE Initialize parameter vectors $\psi, \bar{\psi}, \theta, \phi$.
    \STATE Set prediction length $l_{f}$, update frequency $\gamma_{f}$
    \FOR{each iteration}
        \FOR{each environment step}
            \STATE Sample action $a_t \sim \pi_\theta(a_t|s_t)$.
            \STATE Sample next state $s_{t+1} \sim p(s_{t+1}|s_t, a_t)$.
            \STATE $D \gets D \cup \{(s_t, a_t, r(s_t, a_t), s_{t+1})\}$.
        \ENDFOR 
        \IF {iteration \% $l_f =0$} 
            \STATE $\tilde{Q} = \texttt{ForQ}(D)$.
        \ENDIF 
        \FOR{each gradient step}
            \STATE $\psi \gets \psi - \lambda_\psi \nabla_\psi J_V(\psi)$.
            \STATE $\theta_i \gets \theta_i - \lambda_Q \nabla_{\theta_i} J_Q(\theta_i)$ for $i \in \{1,2\}$.
            \STATE $\bar{\psi} \gets \tau_s \psi + (1 - \tau_s)\bar{\psi}$.
            \IF {iteration \% $l_f \leq l_f \gamma_{f}$}
                \STATE $\phi \gets \phi - \lambda_\pi \nabla_\phi \tilde{J}_\pi(\phi)$.
            \ENDIF
        \ENDFOR
    \ENDFOR
\end{algorithmic}
\label{alg2}
\end{algorithm}

\section{Experiments}
\label{appendix:experiments}
\subsection{Environments and experiments details}

\textbf{Goal switching cliffword}

The environment is $12  \times 3$ tabular MDP where $(0,2)$ is a fixed initial state (blue point), and the possible goal points are $(11,0)$ and $(11,2)$ (for the $x,y$ axis, see Figure \ref{fig:experiment_goalswithcingcliff} (a)). The agent executes $4$ actions (up, left, right, down). If the agent reaches the restart states ($(1,2),(2,2),...,(10,2)$ and $(1,0),(2,0),...,(10,0)$), denoted by yellow points, then the agent goes back to the initial state with a failure reward $-100$. If the agent reaches the goal point, then it receives the success reward $+100$. For taking every step (for every time the agent executes an action), the agent receives a step reward of $-100$. 

For experiments, we use the $Q$-learning algorithm \cite{watkins1992q}. In Figure \ref{fig:experiment_goalswithcingcliff} (b), we denote ``reactive'' label as $Q$-learning algorithm proposed by \cite{watkins1992q} and ``future $Q$'' label as a method that combines $Q$-learning algorithm to evaluate the current policy and use future $Q$ estimator to compute future policy that was proposed in section \ref{Future $Q$ value estimator}. We set the maximum number of steps as $100$. The experiments have been carried out by changing hyperparameters of $Q$-learning: step size $\alpha$ and $\epsilon$ from the $\epsilon$-greedy method. We have done experiments with different $(\alpha, \epsilon) = (0.05,0.05), (0.1, 0.1), (0.1, 0.05),(0.2,0.1), (0.2,0.05), (0.3,0.1)$.

\textbf{Swimmer, Halfcheetah}

The Swimmer and Halfcheetah environments share the same reward function at step $h$ as \(r_{h} = r^{(1)}_{h} + r^{(2)}_{h} + r^{(3)}_{h}\). It comprises a healthy reward (\(r^{(1)}_{h}\)), a forward reward (\(r^{(2)}_{h} = k_{f} \frac{x_{h+1} - x_{h}}{\Delta t_{frame}}, k_{f}>0\)), and a control cost (\(r^{(3)}_{h}\)). We modify the environment to be non-stationary by the agent's desired velocity changes as time goes by. Specifically, we modify the forward reward \(r^{(2)}_{h}\) varies as \(r^{(2)}_{h} = - \left| k_{f} \frac{x_{h+1} - x_{h}}{\Delta t_{frame}} - v_d(t) \right| \), with \(v_d(t) = a\sin(wt)\) and \(t\) representing the episode. Here, \(a,w > 0\) are constants.

For our experiments, we varied hyperparameters such as learning rates $\lambda_\pi \in \{ 0.0001, 0.0003, 0.0005, 0.0007\}$, soft update parameters $\tau_s \in \{ 0.001, 0.005,0.003\}$ and the entropy regularization parameters $\{ 0.01, 0.03, 0.1\}$ and also experimented with different prediction lengths $l_f \in \{ 5,15,20 \}$. We selected the average reward per episode as the performance metric, in line with the definition of dynamic regret. For given hyperparameters, we compare the average reward between FSAC and SAC for different update frequencies $\gamma_f \in \{ 0.1,0.2, \ldots,1.0\}$. The experiments were conducted in two different Mujoco environments: HalfCheetah and Swimmer (see Figures \ref{fig:experiments_hc} and \ref{fig:experiments_swimmer}). In Figures \ref{fig:experiments_hc} and \ref{fig:experiments_swimmer}, error bars denote 0.5 standard deviations. 

\subsection{Results}
\label{appendix:Results}
In this subsection, we have elaborated on the results of the experiment on Halfcheetah and Swimmer. Note that Figures \ref{fig:experiments_hc_fullreward},\ref{fig:experiments_hc_compareFSACSAC_lf5} and \ref{fig:experiments_hc_compareFSACSAC_lf20} are detailed results for Figure \ref{fig:experiments_hc} of the main paper, and Figures \ref{fig:experiments_swimmer_fullreward},\ref{fig:experiments_swimmer_compareFSACSAC_lf5} and \ref{fig:experiments_swimmer_compareFSACSAC_lf15} are detailed result for Figure \ref{fig:experiments_swimmer} of the main paper. Figures \ref{fig:experiments_hc_fullreward} and \ref{fig:experiments_swimmer_fullreward} show the reward return per episode for different update frequencies $\gamma_f \in \{ 0.1,0.2,0.3,0.4,0.5,0.6,0.7,0.8,0.9,1.0\}$. Figures \ref{fig:experiments_hc_compareFSACSAC_lf5},\ref{fig:experiments_hc_compareFSACSAC_lf20},\ref{fig:experiments_swimmer_compareFSACSAC_lf5} and \ref{fig:experiments_swimmer_compareFSACSAC_lf15} compare the FSAC and SAC reward return per episode. Note that the plotted lines are mean rewards calculated over 36 different hyperparameters (learning rates $\lambda_\pi \in \{ 0.0001, 0.0003, 0.0005, 0.0007\}$, soft update parameters $\tau_s \in \{ 0.001, 0.005,0.003\}$ and the entropy regularization parameters $\{ 0.01, 0.03, 0.1\}$).
\begin{figure}[ht]
    \centering
    \subfigure[]{\includegraphics[width=0.45\textwidth]{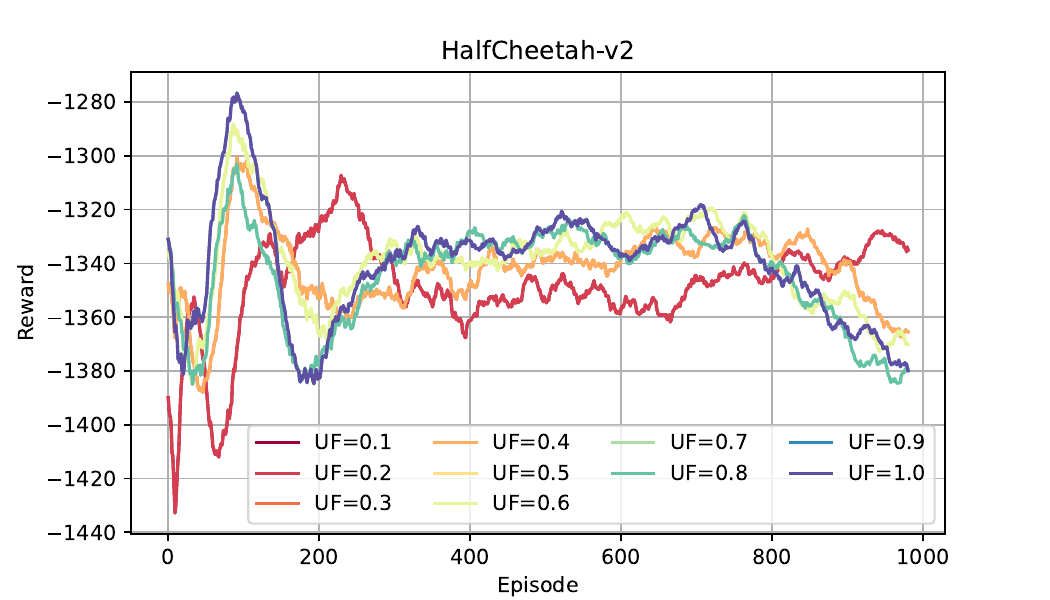}}
    \subfigure[]{\includegraphics[width=0.45\textwidth]{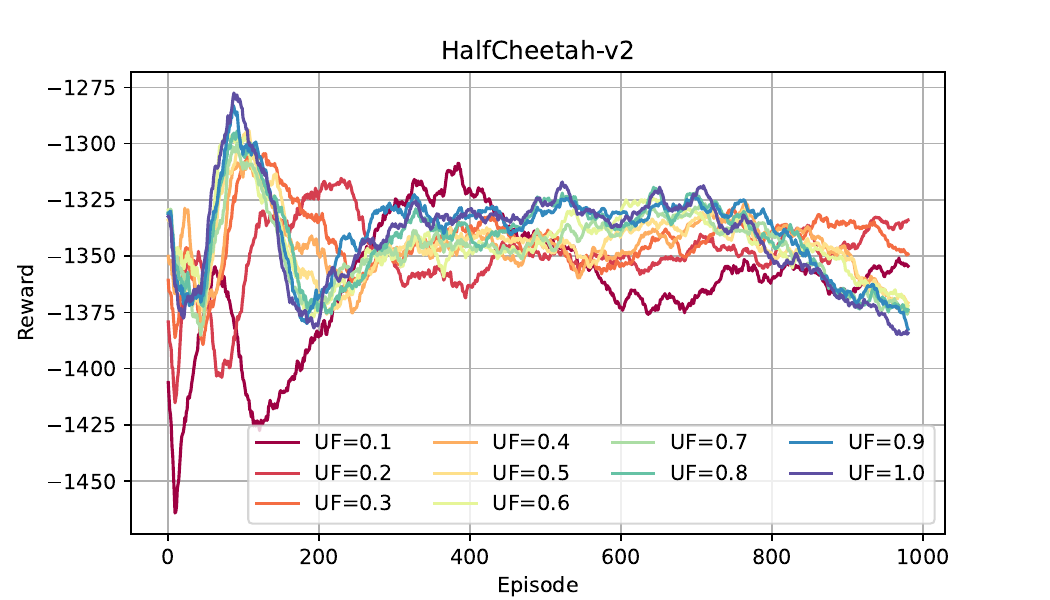}}
    \caption{Reward per episode in the Halfcheetah environment for various update frequencies $\gamma_f \in \{ 0.1,0.2,0.3,0.4,0.5,0.6,0.7$ $,0.8,0.9,1.0\}$. The plotted lines represent the mean reward across 36 different hyperparameters. (a) For $l_f=5$. (b) For $l_f=20$.}
    \label{fig:experiments_hc_fullreward} 
\end{figure}
\begin{figure}[ht]
    \centering
    \subfigure[]{\includegraphics[width=0.45\textwidth]{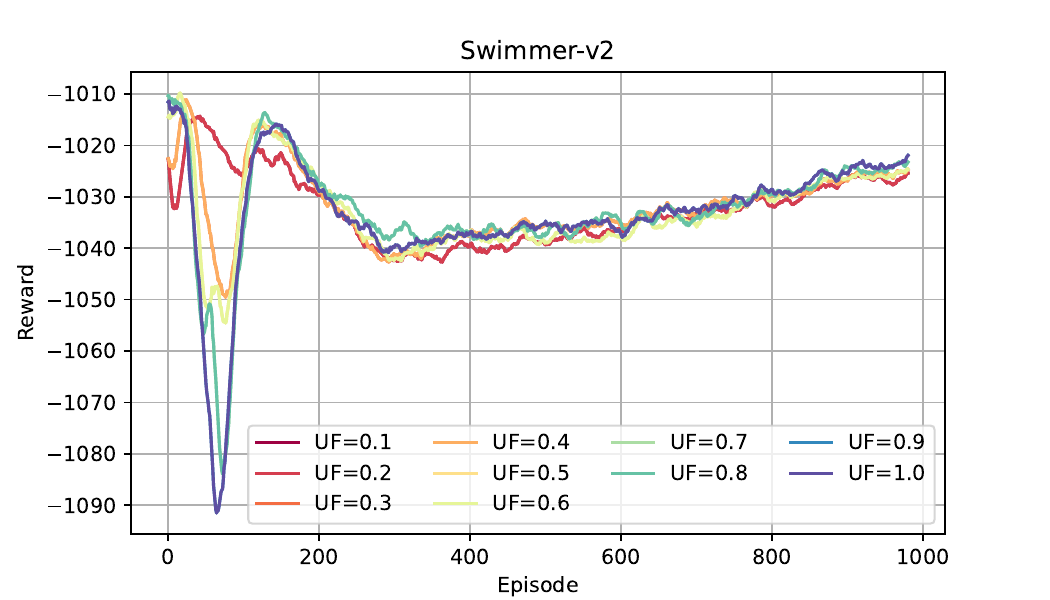}}
    \subfigure[]{\includegraphics[width=0.45\textwidth]{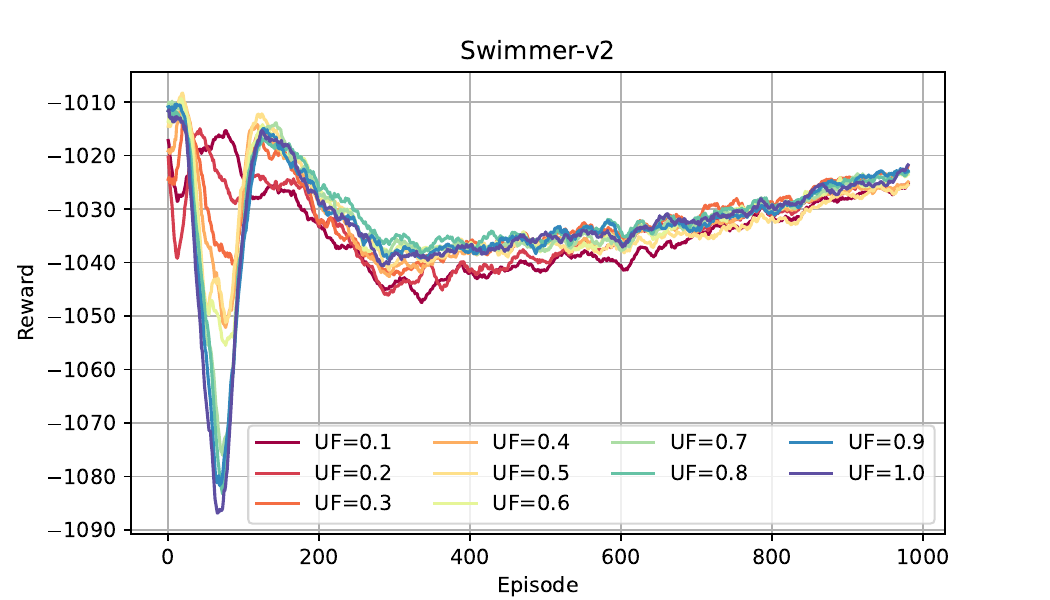}}
    \caption{Reward per episode in the Swimmer environment for various update frequencies $\gamma_f \in \{ 0.1,0.2,0.3,0.4,0.5,0.6,0.7$ $,0.8,0.9,1.0\}$. The plotted lines represent the mean reward across 36 different hyperparameters. (a) For $l_f=5$. (b) For $l_f=15$.}
    \label{fig:experiments_swimmer_fullreward}
\end{figure}
\newpage
\begin{figure}[ht]
    \centering
    \subfigure[]{\includegraphics[width=0.19\textwidth]{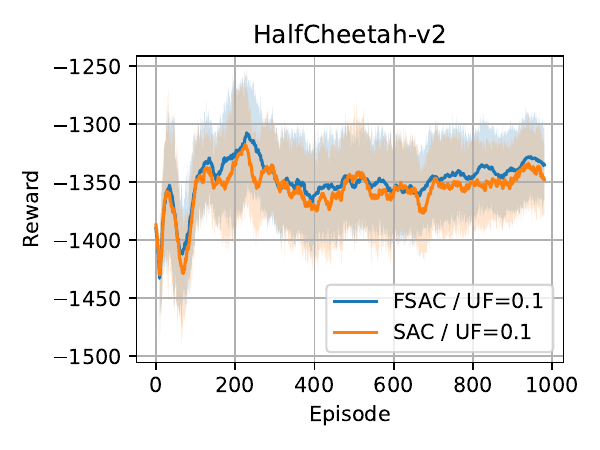}}
    \subfigure[]{\includegraphics[width=0.19\textwidth]{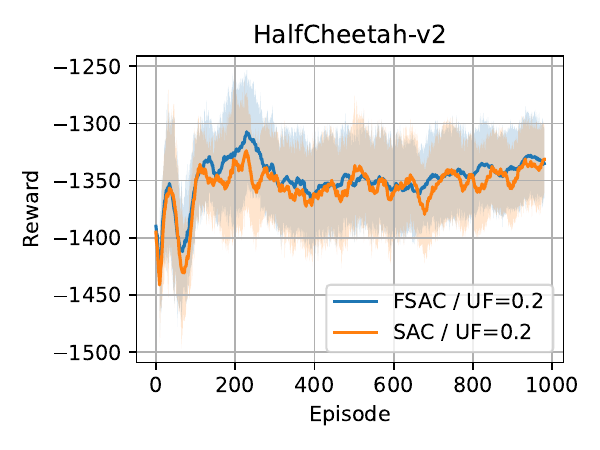}}
    \subfigure[]{\includegraphics[width=0.19\textwidth]{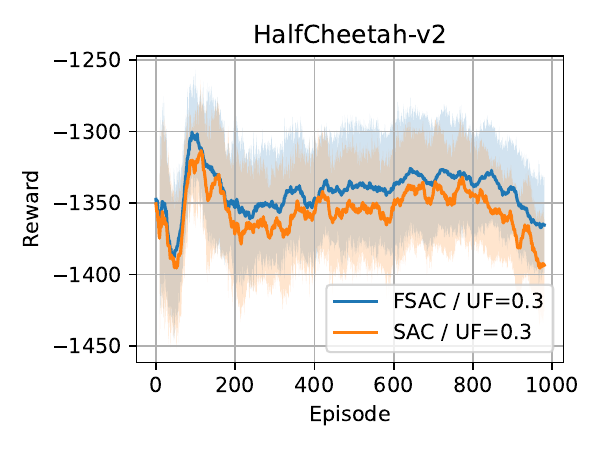}}
    \subfigure[]{\includegraphics[width=0.19\textwidth]{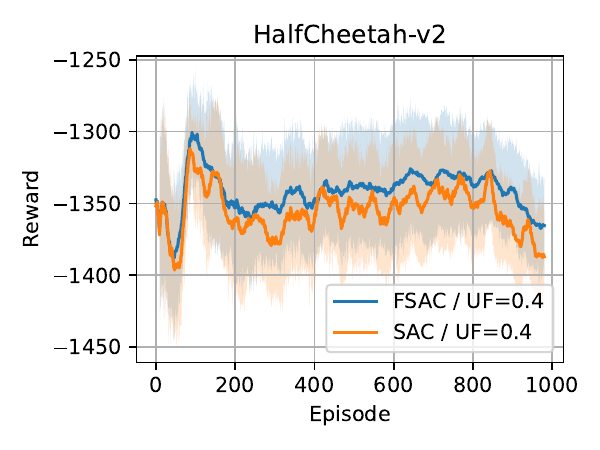}}
    \subfigure[]{\includegraphics[width=0.19\textwidth]{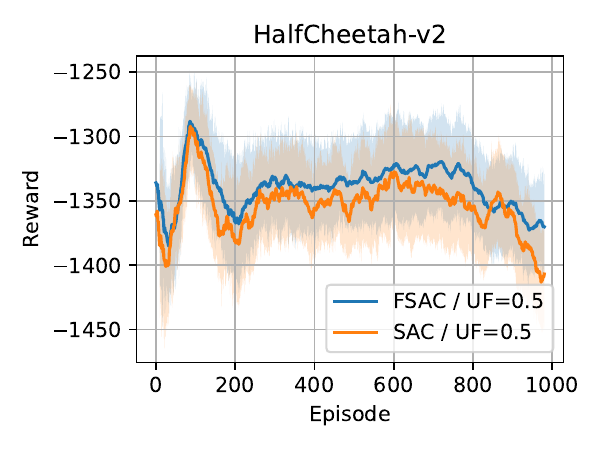}}
    \subfigure[]{\includegraphics[width=0.19\textwidth]{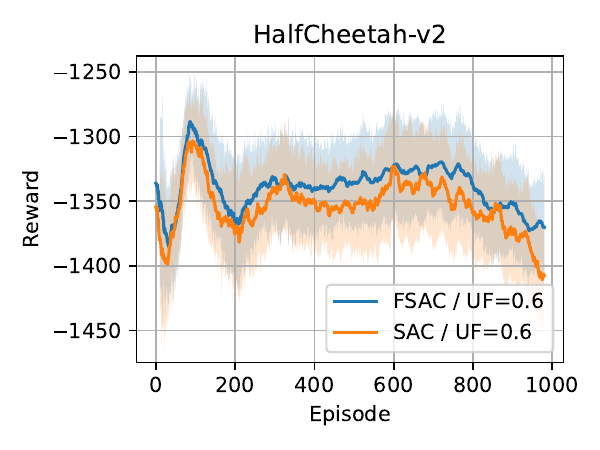}}
    \subfigure[]{\includegraphics[width=0.19\textwidth]{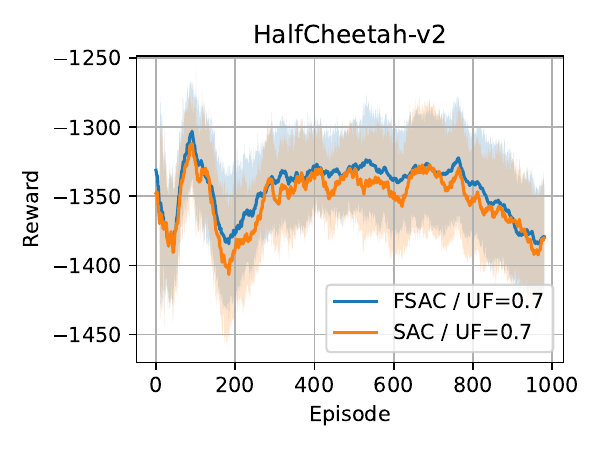}}
    \subfigure[]{\includegraphics[width=0.19\textwidth]{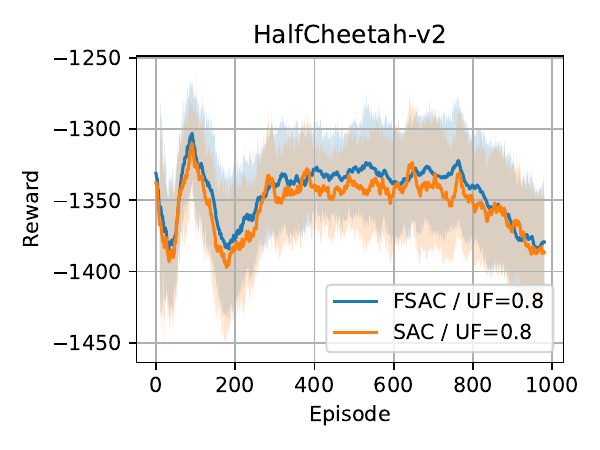}}
    \subfigure[]{\includegraphics[width=0.19\textwidth]{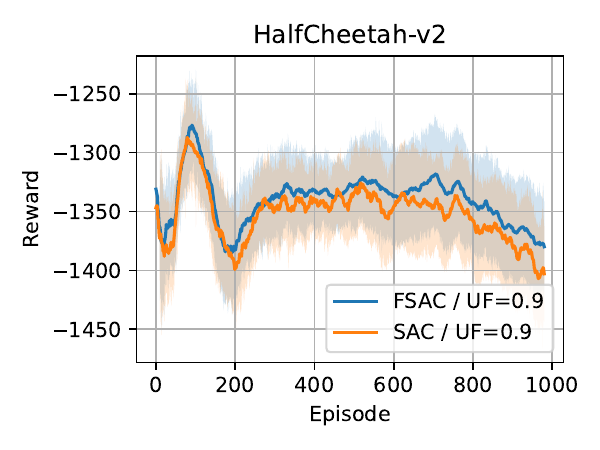}}
    \subfigure[]{\includegraphics[width=0.19\textwidth]{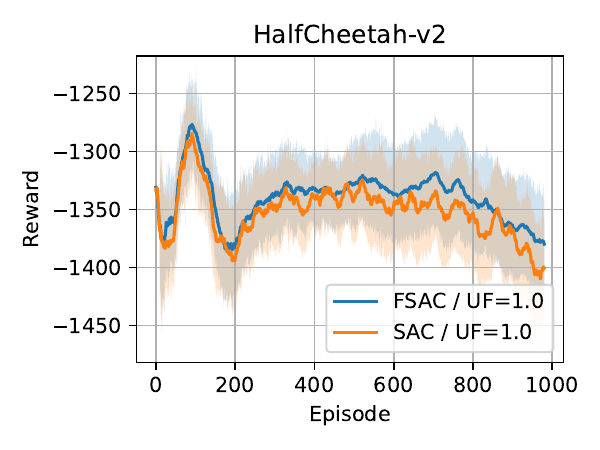}}
    \caption{Reward per episode for Halfcheetah environment when $l_f=5$. The blue lines are FSAC, and the orange lines are SAC. The shaded areas are 0.5 standard deviations over 36 different hyperparameter results. (a) $\gamma_f=0.1$. (b) $\gamma_f=0.2$. (c) $\gamma_f=0.3$. (d) $\gamma_f=0.4$. (e) $\gamma_f=0.5$. (f) $\gamma_f=0.6$. (g) $\gamma_f=0.7$. (h) $\gamma_f=0.8$. (i) $\gamma_f=0.9$. (j) $\gamma_f=1.0$.}
    \label{fig:experiments_hc_compareFSACSAC_lf5}
\end{figure}

\begin{figure}[ht]
    \centering
    \subfigure[]{\includegraphics[width=0.19\textwidth]{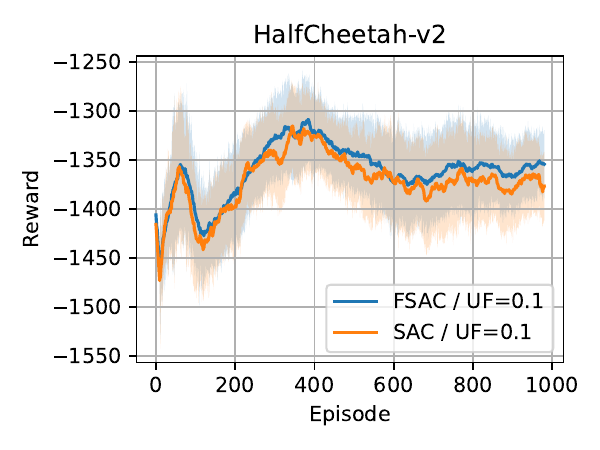}}
    \subfigure[]{\includegraphics[width=0.19\textwidth]{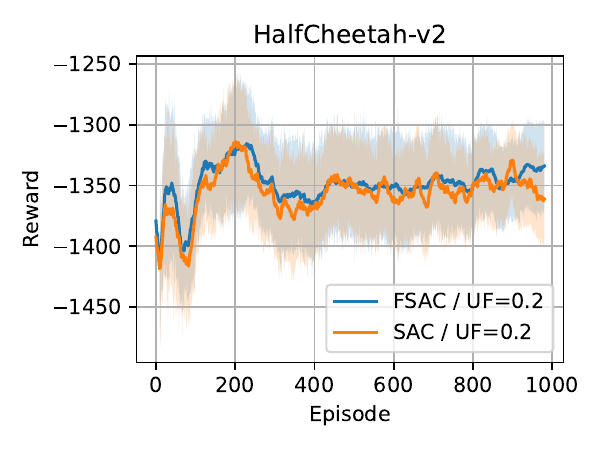}}
    \subfigure[]{\includegraphics[width=0.19\textwidth]{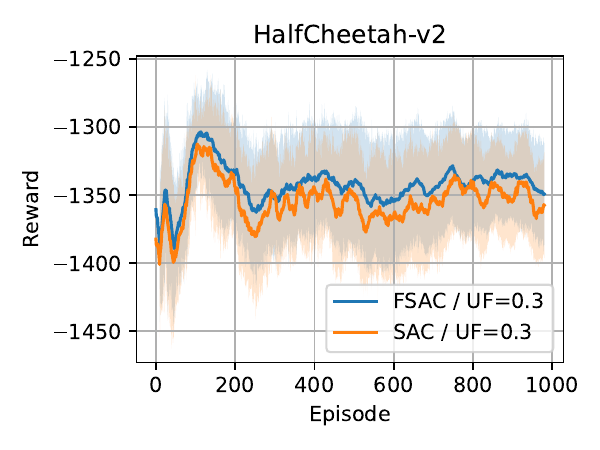}}
    \subfigure[]{\includegraphics[width=0.19\textwidth]{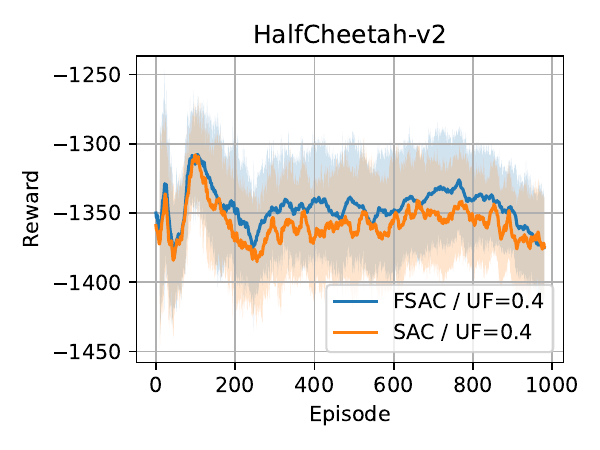}}
    \subfigure[]{\includegraphics[width=0.19\textwidth]{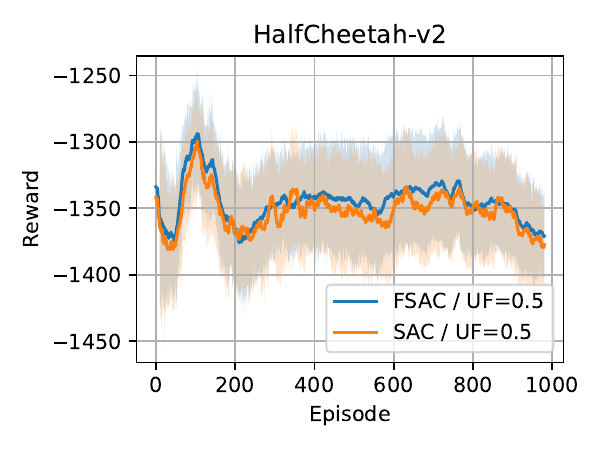}}
    \subfigure[]{\includegraphics[width=0.19\textwidth]{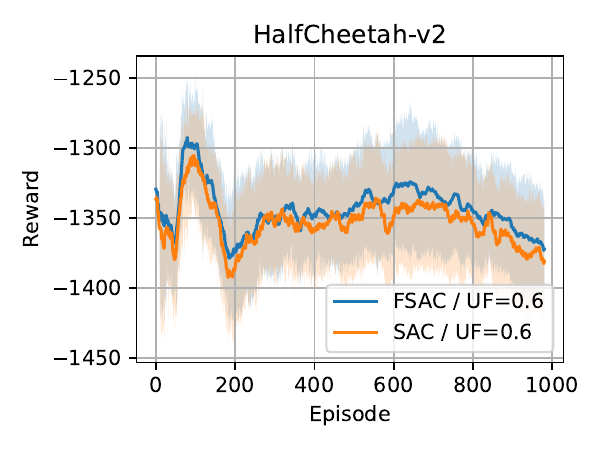}}
    \subfigure[]{\includegraphics[width=0.19\textwidth]{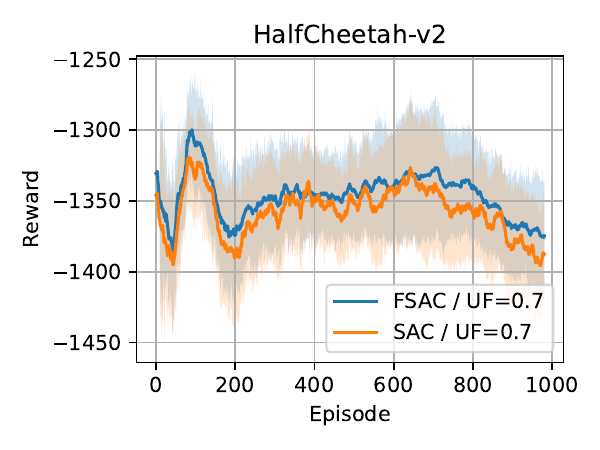}}
    \subfigure[]{\includegraphics[width=0.19\textwidth]{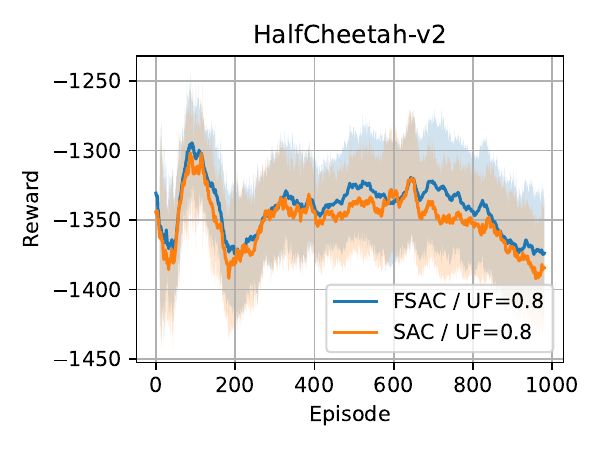}}
    \subfigure[]{\includegraphics[width=0.19\textwidth]{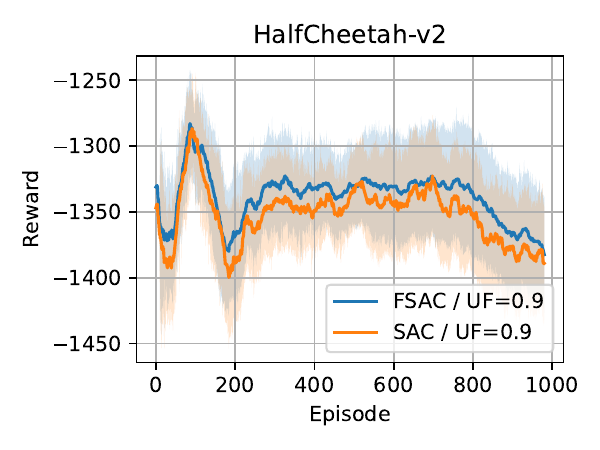}}
    \subfigure[]{\includegraphics[width=0.19\textwidth]{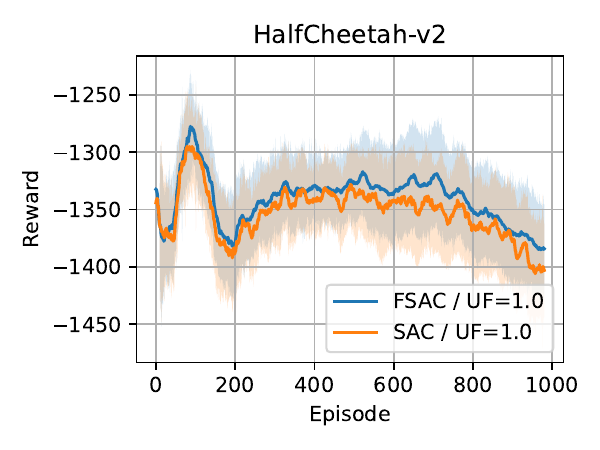}}
    \caption{Reward per episode for Halfcheetah environment when $l_f=20$. The blue lines are FSAC, and the orange lines are SAC. The shaded areas are 0.5 standard deviations over 36 different hyperparameter results. (a) $\gamma_f=0.1$. (b) $\gamma_f=0.2$. (c) $\gamma_f=0.3$. (d) $\gamma_f=0.4$. (e) $\gamma_f=0.5$. (f) $\gamma_f=0.6$. (g) $\gamma_f=0.7$. (h) $\gamma_f=0.8$. (i) $\gamma_f=0.9$. (j) $\gamma_f=1.0$.}
    \label{fig:experiments_hc_compareFSACSAC_lf20}
\end{figure}

\newpage 
\begin{figure}[ht]
    \centering
    \subfigure[]{\includegraphics[width=0.19\textwidth]{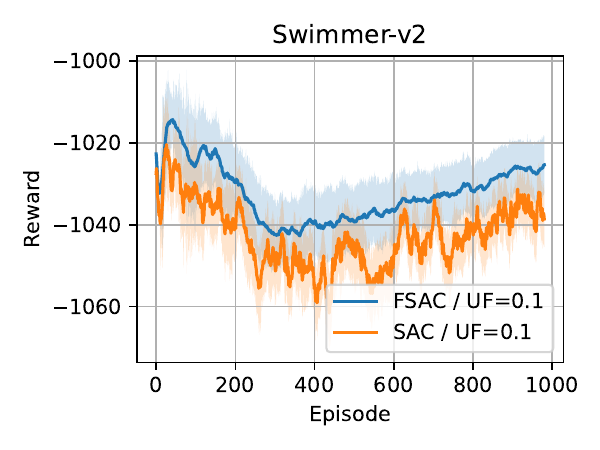}}
    \subfigure[]{\includegraphics[width=0.19\textwidth]{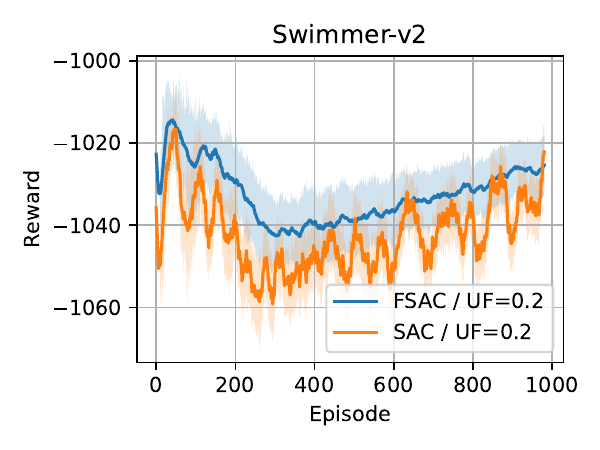}}
    \subfigure[]{\includegraphics[width=0.19\textwidth]{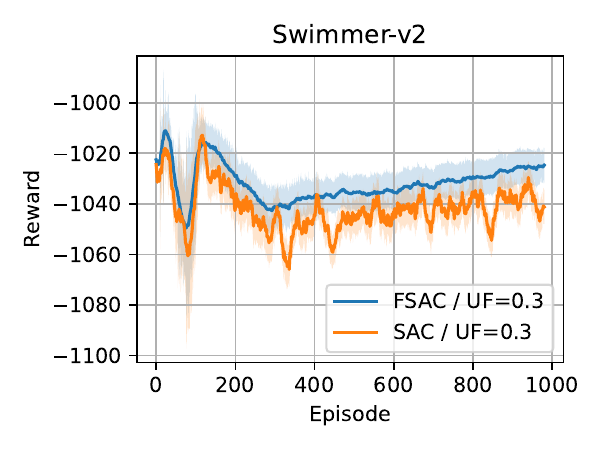}}
    \subfigure[]{\includegraphics[width=0.19\textwidth]{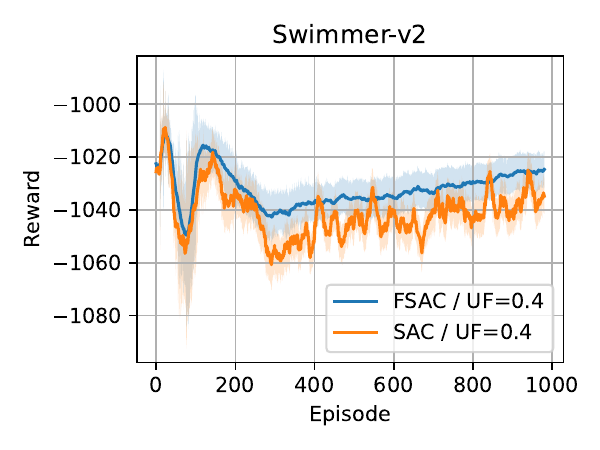}}
    \subfigure[]{\includegraphics[width=0.19\textwidth]{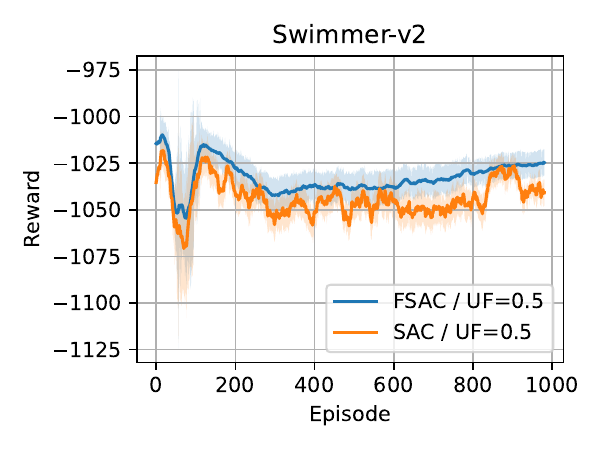}}
    \subfigure[]{\includegraphics[width=0.19\textwidth]{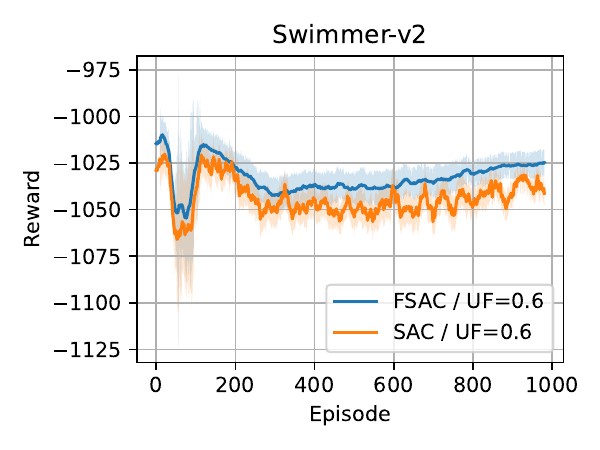}}
    \subfigure[]{\includegraphics[width=0.19\textwidth]{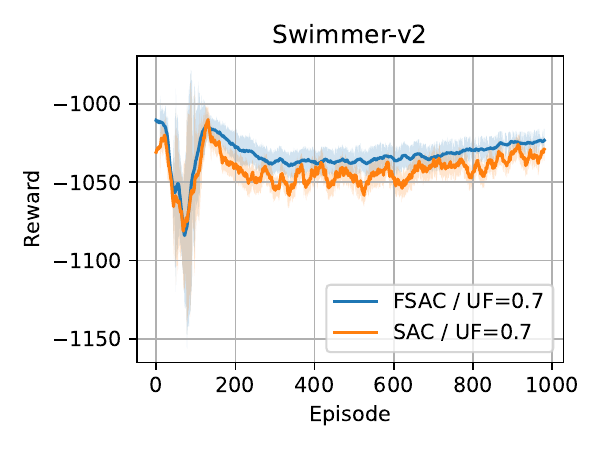}}
    \subfigure[]{\includegraphics[width=0.19\textwidth]{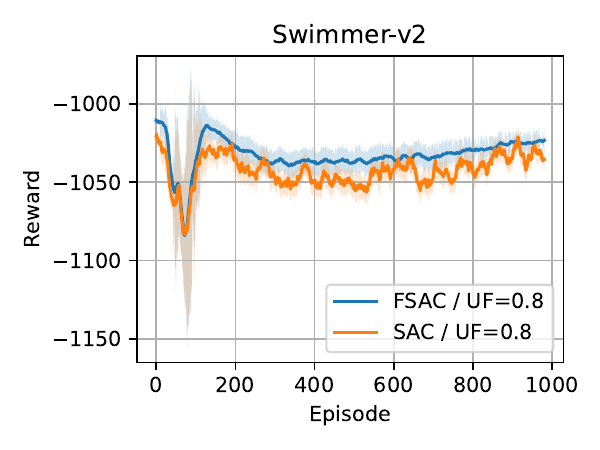}}
    \subfigure[]{\includegraphics[width=0.19\textwidth]{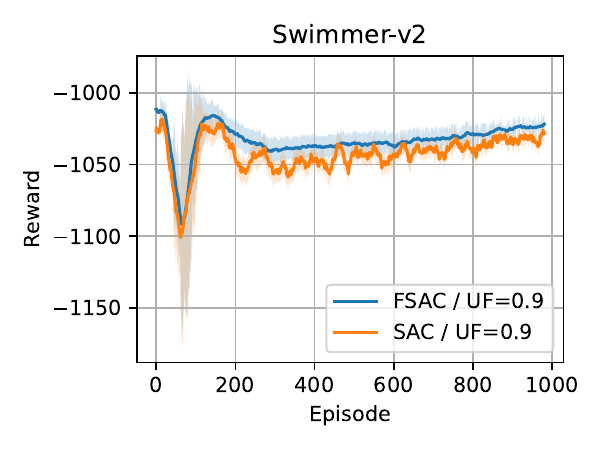}}
    \subfigure[]{\includegraphics[width=0.19\textwidth]{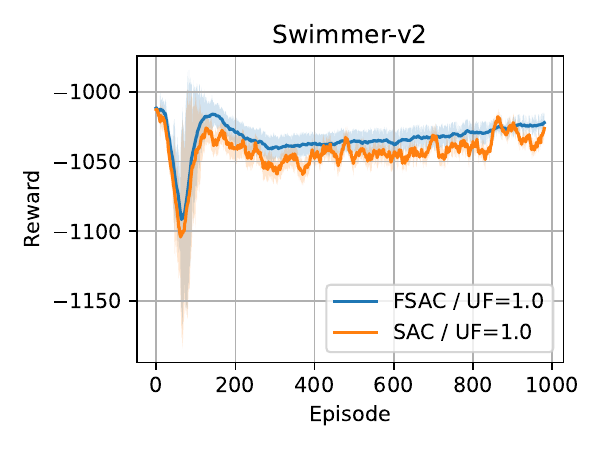}}
    \caption{Reward per episode for Swimmer environment when $l_f=5$. The blue lines are FSAC, and the orange lines are SAC. The shaded areas are 0.5 standard deviations over 36 different hyperparameter results. (a) $\gamma_f=0.1$. (b) $\gamma_f=0.2$. (c) $\gamma_f=0.3$. (d) $\gamma_f=0.4$. (e) $\gamma_f=0.5$. (f) $\gamma_f=0.6$. (g) $\gamma_f=0.7$. (h) $\gamma_f=0.8$ (i) $\gamma_f=0.9$. (j) $\gamma_f=1.0$.}
    \label{fig:experiments_swimmer_compareFSACSAC_lf5}
\end{figure}

\begin{figure}[ht]
    \centering
    \subfigure[]{\includegraphics[width=0.19\textwidth]{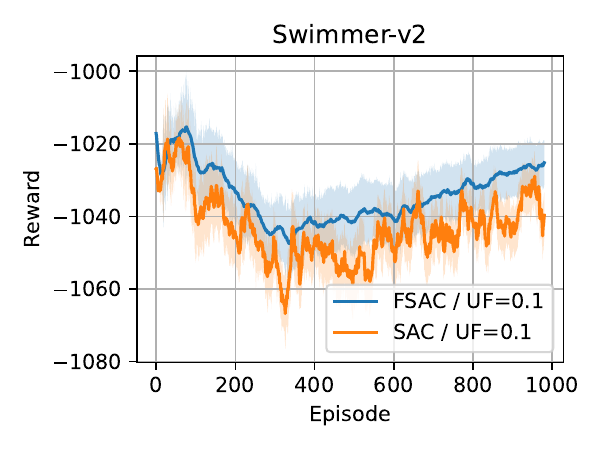}}
    \subfigure[]{\includegraphics[width=0.19\textwidth]{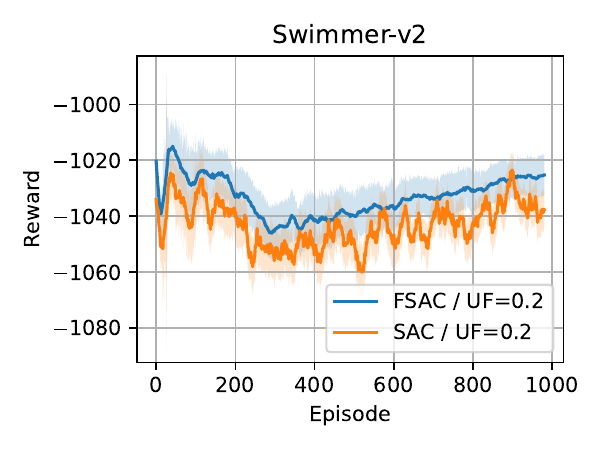}}
    \subfigure[]{\includegraphics[width=0.19\textwidth]{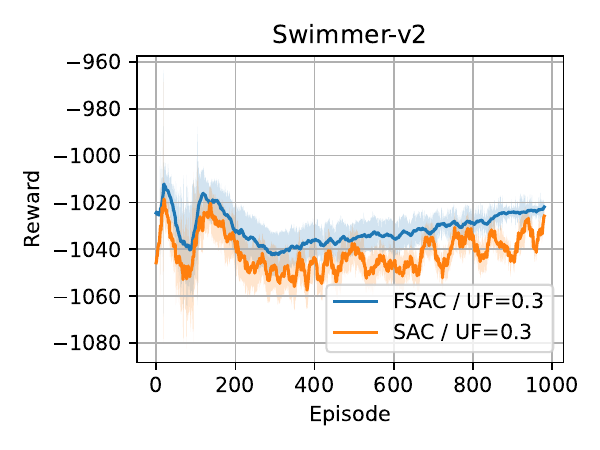}}
    \subfigure[]{\includegraphics[width=0.19\textwidth]{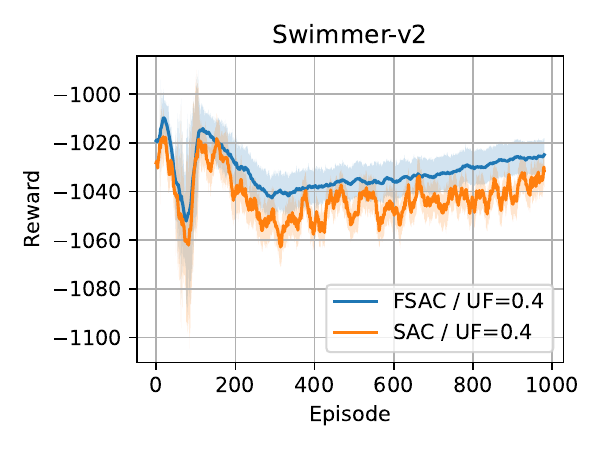}}
    \subfigure[]{\includegraphics[width=0.19\textwidth]{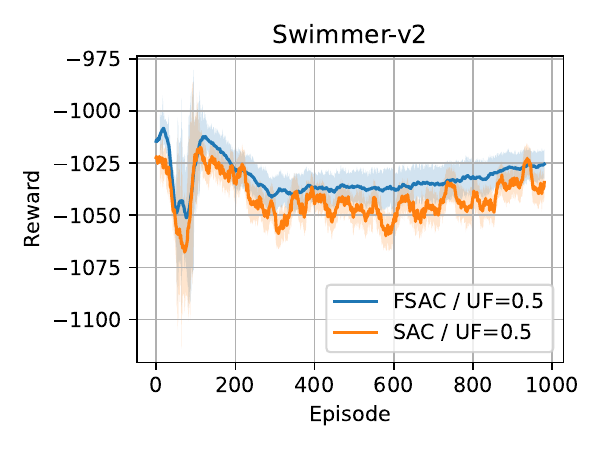}}
    \subfigure[]{\includegraphics[width=0.19\textwidth]{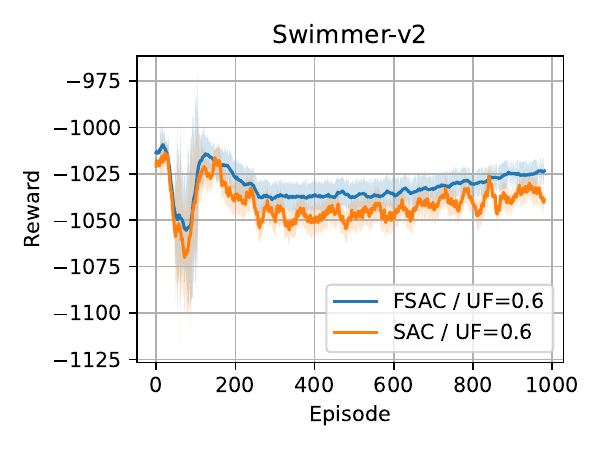}}
    \subfigure[]{\includegraphics[width=0.19\textwidth]{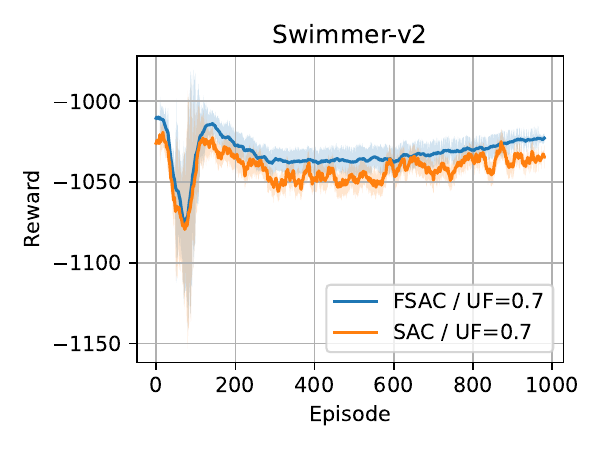}}
    \subfigure[]{\includegraphics[width=0.19\textwidth]{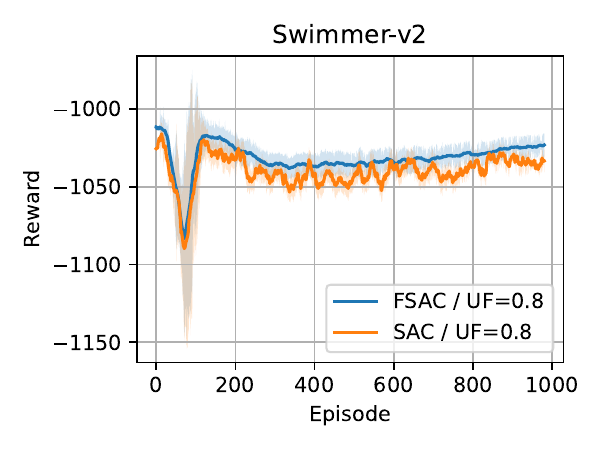}}
    \subfigure[]{\includegraphics[width=0.19\textwidth]{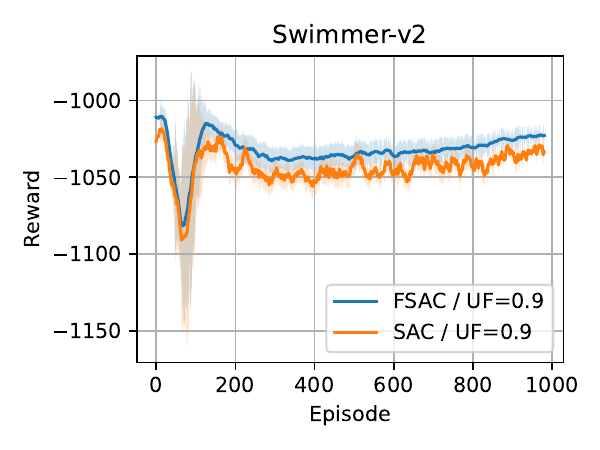}}
    \subfigure[]{\includegraphics[width=0.19\textwidth]{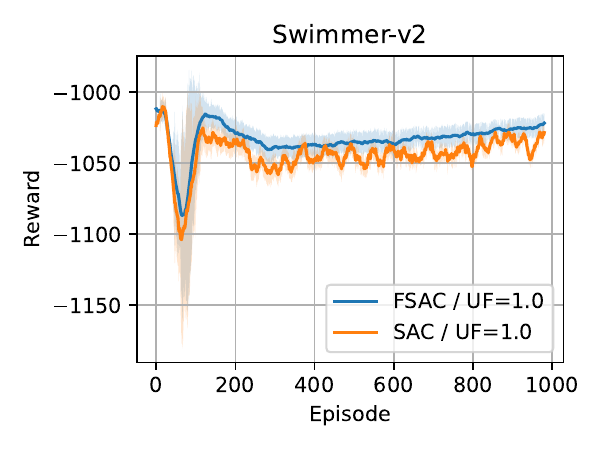}}
    \caption{Reward per episode for Swimmer environment when $l_f=15$. The blue lines are FSAC, and the orange lines are SAC. The shaded areas are 0.5 standard deviations over 36 different hyperparameter results. (a) $\gamma_f=0.1$. (b) $\gamma_f=0.2$. (c) $\gamma_f=0.3$. (d) $\gamma_f=0.4$. (e) $\gamma_f=0.5$. (f) $\gamma_f=0.6$. (g) $\gamma_f=0.7$. (h) $\gamma_f=0.8$. (i) $\gamma_f=0.9$. (j) $\gamma_f=1.0$.}
    \label{fig:experiments_swimmer_compareFSACSAC_lf15}
\end{figure}

\section{Proofs}
\begin{proof}[Proof of Proposition \ref{prop1}]
\begin{align}
    || \widetilde{Q}_{t_{m+1}} - Q^*_{t_{m+1}} ||_{\infty} &= \left|\left| \sum_{t=t_m-l_p+1}^{t_m} w_t \left(\widehat{Q}_t - Q^*_{t_{m+1}}  \right) \right|\right|_{\infty}+  \left|\left| \sum_{t=t_m-l_p+1}^{t_m} \left( w_t -1 \right) Q^*_{t_{m+1}} \right|\right|_{\infty} \nonumber \\
    &\leq  \sum_{t=t_m-l_p+1}^{t_m} |w_t|  \left( || Q^*_{t} - Q^*_{t_{m+1}} ||_{\infty} + || Q^*_{t} - \widehat{Q}_{t} ||_{\infty}  \right) +  \sum_{t=t_m-l_p+1}^{t_m} |w_t-1| \left|\left|Q^*_{t_{m+1}}\right|\right|_\infty \nonumber\\
    &\leq  \sqrt{ \sum_{t=t_m-l_p+1}^{t_m} |w_t|^2 } \sqrt{  \sum_{t=t_m-l_p+1}^{t_m} \left( || Q^*_{t} - Q^*_{t_{m+1}} ||_{\infty}+ || Q^*_{t} - \widehat{Q}_{t} ||_{\infty} \right)^2} \nonumber \\
    & \quad +  \left( \sum_{t=t_m-l_p+1}^{t_m} |w_t|+ l_p \right) \left|\left|Q^*_{t_{m+1}}\right|\right|_\infty \nonumber\\
    &\leq L \cdot \sqrt{   \sum_{t=t_m-l_p+1}^{t_m} \left( || Q^*_{t} - Q^*_{t_{m+1}} ||^2_{\infty} + 2|| Q^*_{t} - Q^*_{t_{m+1}} ||_{\infty}|| Q^*_{t} - \widehat{Q}_{t} ||_{\infty} || Q^*_{t} - \widehat{Q}_{t} ||^2_{\infty} \right)} \nonumber\\
    & \quad +  \left( l_p \sqrt{\sum_{t=t_m-l+1}^{t_m} |w_t|^2}+ l_p \right) \left( \frac{1-\gamma^H}{1-\gamma}r_{max}  \right). \label{prop:eq1}
\end{align}
We use Lemma \ref{lemma:optimalQGap} to conclude that 
$$|| Q^*_{t} - Q^*_{t_{m+1}} ||_{\infty} \leq \frac{1-\gamma^H}{1-\gamma} \left( B_r(t,t_{m+1}) + \frac{r_{max}}{1-\gamma} B_p(t,t_{m+1})\right).$$ Moreover, the assumption 
$$|| Q^*_{t} - \widehat{Q}_{t} ||_{\infty} \leq \epsilon_t$$ holds. As a result,

\begin{align*}
    & \left|\left| \widetilde{Q}_{t_{m+1}} - Q^*_{t_{m+1}} \right|\right|_{\infty} \\
    & \leq L \sqrt{ \sum_{t=t_m-l+1}^{t_m} 
    \begin{array}{l}
    \left[ \left(\frac{1-\gamma^H}{1-\gamma} \left( B_r(t,t_{m+1}) + \frac{r_{max}}{1-\gamma} B_p(t,t_{m+1}) \right)\right)^2 \right. \\
    \quad \left. + 2\frac{1-\gamma^H}{1-\gamma} \left( B_r(t,t_{m+1}) + \frac{r_{max}}{1-\gamma} B_p(t,t_{m+1}) \right) \epsilon_t + \epsilon^2_t \right]
    \end{array}
    } \\
    & \quad + l_p (L+1) \left( \frac{1-\gamma^H}{1-\gamma} r_{max} \right).
\end{align*}

To simplify the expression, define $u_t :=  \frac{1-\gamma^H}{1-\gamma} \left( B_r(t,t_{m+1}) + \frac{r_{max}}{1-\gamma} B_p(t,t_{m+1})\right)$. Then, the inequality \eqref{prop:eq1} can be rewritten in a simpler form as follows: 
\begin{align*}
    \left|\left| \widetilde{Q}_{t_{m+1}} - Q^*_{t_{m+1}} \right|\right|_{\infty} &\leq L \sqrt{  \sum_{t=t_m-l+1}^{t_m} 2 \left(\max(u_t,\epsilon_t)\right)^2}+l_p(L+1) \left( \frac{1-\gamma^H}{1-\gamma}r_{max}  \right) \\
    &\leq \sqrt{2}Ll_p \max_{t \in [t_m-l+1,t_m]} \left(\max(u_t,\epsilon_t)\right) +l_p(L+1) \left( \frac{1-\gamma^H}{1-\gamma}r_{max}  \right).
\end{align*}

\begin{proof}[Proof of Proposition \ref{prop2}]
    Refer to Theorem 7 of \cite{pmlr-v125-qu20a}.
\end{proof}

\end{proof}
\begin{proof}[Proof of Lemma \ref{lemma1}]

The policy update term is divided into three terms:

\begin{align*}
    \sum_{t \in \mathcal{G}_m} \left( V_t^* - V_t^{\pi_t} \right) &= \sum_{g=0}^{G_m-1}  \left( V_{t_m +g}^* - V_{t_m +g}^{\pi_{t_m+g}} \right)  \\ 
    &= \sum_{g=0}^{G_m-1}  \left( \underbrace{\left( V_{t_m +G_m -1}^* - V_{t_m + G_m -1}^{\pi_{t_m+g}} \right)}_{\text{(1-I)}} + \underbrace{\left( V_{t_m +G_m -1}^{\pi_{t_m+g}} - V_{t_m +g}^{\pi_{t_m+g}} \right)}_{\text{(1-II)}} + \underbrace{\left( V_{t_m + g}^* - V_{t_m + G_m -1}^* \right)}_{\text{(1-III)}} \right).
\end{align*}
Note that the term (1-I), the term (1-II), and the term (1-III) are upper bounded by Lemma \ref{lemma:NPGconvergence}, Corollary \ref{cor:optimalValueGap}, and Lemma \ref{lemma:valuegap_samepolicy}. 

For any $g \in [0,G_m-1]$ and for any $s \in \mathcal{S}$, one can write: 
\begin{itemize}
    \item $V^*_{t_m + G_m-1} - V^{\pi_{t_m+g}}_{t_m + G_m-1} \leq (\gamma+2)((1- \eta \tau)^{g}C^\prime) + \frac{2(\gamma+2)}{1-\gamma}\left( 1 + \frac{\gamma}{\eta \tau} \right)\cdot \epsilon_f + \frac{2\tau \log{|\mathcal{A}|}}{1-\gamma}$
    \item $ V_{t_m +G_m -1}^{\pi_{t_m+g}}(s) - V_{t_m +g}^{\pi_{t_m+g}}(s) \leq \frac{1- \gamma^H}{1-\gamma} \cdot B_r(t_m+g,t_m+G_m-1) + \frac{\gamma}{1-\gamma} \cdot \left( \frac{1-\gamma^H}{1-\gamma} - \gamma^{H-1} H \right) \cdot B_p(t_m+g,t_m+G_m-1)$
    \item $V^*_{t_m+g}(s) - V^*_{t_m + G_m -1}(s)  \leq \frac{1-\gamma^H}{1-\gamma} \left( B_r(t_m+g,t_m+G_m-1) + \frac{r_{max}}{1-\gamma} B_p(t_m+g,t_m+G_m-1)\right)$
\end{itemize}
where $C^\prime = || Q^*_\tau - Q^{t_m}_\tau||_\infty + 2\tau (1- \frac{\eta \tau}{1- \gamma} || \log \pi^*_\tau - \log \pi^{t_m}_\tau||_\infty)$. Now, taking the summation over $g = 0,...,G_m-1$ gives rise to
\begin{align*}
    &\sum_{t \in \mathcal{G}_m}  \left( V_t^* - V_t^{\pi_t} \right) \\ 
    =& (\gamma +2) C^\prime \frac{1-(1-\eta \tau)^{G_m}}{\eta \tau} + G_m \cdot \left( \frac{2(\gamma+2)}{1-\gamma}\left( 1 + \frac{\gamma}{\eta \tau} \right)\cdot \epsilon_f + \frac{2\tau \log{|\mathcal{A}|}}{1-\gamma} \right) \\ 
    &+ \frac{1- \gamma^H}{1-\gamma} \cdot \left( \sum_{g=0}^{G_m-1} B_r(t_m+g,t_m+G_m-1) \right) \\ 
    &+ \frac{\gamma}{1-\gamma} \cdot \left( \frac{1-\gamma^H}{1-\gamma} - \gamma^{H-1} H \right) \cdot \left( \sum_{g=0}^{G_m-1} B_p(t_m+g,t_m+G_m-1) \right) \\
    &+ \frac{1-\gamma^H}{1-\gamma} \left( \sum_{g=0}^{G_m-1} B_r(t_m+g,t_m+G_m-1) + \frac{r_{max}}{1-\gamma} \sum_{g=0}^{G_m-1} B_p(t_m+g,t_m+G_m-1)\right) \\
    \leq& (\gamma +2) C^\prime \frac{1-(1-\eta \tau)^{G_m}}{\eta \tau} + G_m \cdot \left( \frac{2(\gamma+2)}{1-\gamma}\left( 1 + \frac{\gamma}{\eta \tau} \right)\cdot \epsilon_f + \frac{2\tau \log{|\mathcal{A}|}}{1-\gamma} \right) \\ 
    &+ \frac{2(1- \gamma^H)}{1-\gamma} \cdot \left( \bar{B}_r(t_m,t_m+G_m-1) \right) \\
    & + \left( \frac{\gamma}{1-\gamma} \cdot \left( \frac{1-\gamma^H}{1-\gamma} - \gamma^{H-1} H \right) + \frac{1-\gamma^H}{1-\gamma} \cdot \frac{r_{max}}{1-\gamma} \right) \cdot \left( \bar{B}_r(t_m,t_m+G_m-1)  \right) \\ 
    &= \frac{C_1}{\eta \tau} \cdot \left( 1- (1-\eta \tau)^{G_m} \right) + G_m \cdot \left( C_2 \epsilon_f + C_3\right) + C_4 \bar{B}_r(\mathcal{G}_m) + C_5 \bar{B}_p(\mathcal{G}_m) 
\end{align*}
where $C_1= (\gamma+2)\left(|| Q^*_{t_m} - Q_{t_m}||_\infty + 2\tau (1- \frac{\eta \tau}{1- \gamma} || \log \pi_{t_m}^* - \log \pi_{t_m}||_\infty) \right),C_2 = \frac{2(\gamma+2)}{1-\gamma}\left( 1 + \frac{\gamma}{\eta \tau} \right), C_3 = \frac{2 \tau \log |\mathcal{A}|}{1-\gamma}, C_4= \frac{2(1-\gamma^H)}{1-\gamma},C_5=\frac{\gamma}{1-\gamma} \cdot \left( \frac{1-\gamma^H}{1-\gamma} - \gamma^{H-1} H \right) + \frac{1-\gamma^H}{1-\gamma} \cdot \frac{r_{max}}{1-\gamma}$.
\end{proof}

\begin{proof}[Proof of Lemma \ref{lemma2}]
The policy hold error can be divided into three terms:
\begin{align*}
    \sum_{t\in \mathcal{N}_{m}} (V_t^*-V_t^{\pi_t}) &= \sum_{n=0}^{N_m-1} (V_{t_m+G_m+n}^*-V_{t_m+G_m+n}^{\pi_{t_m+G_m}}) \\ 
    &= \sum_{n=0}^{N_m-1} \left( \underbrace{(V_{t_m+G_m+n}^*-V_{t_m+G_m}^{*})}_{\text{(2-I)}}+ \underbrace{(V_{t_m+G_m}^{*}-V_{t_m+G_m}^{\pi_{t_m+G_m}})}_{\text{(2-II)}}+\underbrace{(V_{t_m+G_m}^{\pi_{t_m+G_m}}-V_{t_m+G_m+n}^{\pi_{t_m+G_m}})}_{\text{(2-III)}} \right).
\end{align*}
The terms (2-I), (2-II) and (2-III) can be bounded using Corollary \ref{cor:optimalValueGap}, Lemma \ref{lemma:NPGconvergence} and Lemma \ref{lemma:valuegap_samepolicy}. Recall that we have defined the time interval $\mathcal{N}_m = [t_m+G_m,t_{m+1})$, where $t_{m+1}=t_m+G_m+N_m$. One can write:
\begin{itemize}
    \item $V_{t_m+G_m+n}^*-V_{t_m+G_m}^{*} \leq \frac{1-\gamma^H}{1-\gamma} \left( B_r(t_m+G_m,t_m+G_m+n) + \frac{r_{max}}{1-\gamma} B_p(t_m+G_m,t_m+G_m+n)\right)$
    \item $V_{t_m+G_m}^{*}-V_{t_m+G_m}^{\pi_{t_m+G_m}} \leq (\gamma+2)((1- \eta \tau)^{G_m}C^\prime) + \frac{2(\gamma+2)}{1-\gamma}\left( 1 + \frac{\gamma}{\eta \tau} \right)\cdot \epsilon_f + \frac{2\tau \log{|\mathcal{A}|}}{1-\gamma}$
    \item $V_{t_m+G_m}^{\pi_{t_m+G_m}}-V_{t_m+G_m+n}^{\pi_{t_m+G_m}} \leq \frac{1- \gamma^H}{1-\gamma} \cdot B_r(t_m+G_m,t_m+G_m+n) + \frac{\gamma}{1-\gamma} \cdot \left( \frac{1-\gamma^H}{1-\gamma} - \gamma^{H-1} H \right) \cdot B_p(t_m+G_m,t_m+G_m+n)$.
\end{itemize}
Now, taking the summation over $n=0,1,...,N_m-1$ leads to
\begin{align*}
    \sum_{t \in \mathcal{N}_m}  \left( V_t^* - V_t^{\pi_t} \right) 
    =& N_m \cdot \left( (\gamma+2)((1- \eta \tau)^{G_m}C^\prime) +  \frac{2(\gamma+2)}{1-\gamma}\left( 1 + \frac{\gamma}{\eta \tau} \right)\cdot \epsilon_f + \frac{2\tau \log{|\mathcal{A}|}}{1-\gamma} \right) \\ 
    &+ \frac{1- \gamma^H}{1-\gamma} \cdot \left( \sum_{n=0}^{N_m-1} B_r(t_m+G_m,t_m+G_m+n) \right) \\ 
    &+ \frac{\gamma}{1-\gamma} \cdot \left( \frac{1-\gamma^H}{1-\gamma} - \gamma^{H-1} H \right) \cdot \left( \sum_{n=0}^{N_m-1} B_p(t_m+G_m,t_m+G_m+n) \right) \\
    &+ \frac{1-\gamma^H}{1-\gamma} \left( \sum_{n=0}^{N_m-1} B_r(t_m+G_m,t_m+G_m+n) + \frac{r_{max}}{1-\gamma} \sum_{n=0}^{N_m-1} B_p(t_m+G_m,t_m+G_m+n)\right) \\
    \leq& N_m \cdot \left( (\gamma+2)((1- \eta \tau)^{G_m}C^\prime) +  \frac{2(\gamma+2)}{1-\gamma}\left( 1 + \frac{\gamma}{\eta \tau} \right)\cdot \epsilon_f + \frac{2\tau \log{|\mathcal{A}|}}{1-\gamma} \right)  \\ 
    &+ \frac{2(1- \gamma^H)}{1-\gamma} \cdot \left(\bar{B}_r(t_m+G_m,t_m+G_m+N_m-1) \right) \\
    &+ \left( \frac{\gamma}{1-\gamma} \cdot \left( \frac{1-\gamma^H}{1-\gamma} - \gamma^{H-1} H \right) + \frac{1-\gamma^H}{1-\gamma} \cdot \frac{r_{max}}{1-\gamma} \right) \cdot \left( \bar{B}_r(t_m+G_m,t_m+G_m+N_m-1) \right) \\ 
    =& N_m \cdot \left( C_1(1-\eta \tau)^{G_m} + C_2 \epsilon_f + C_3\right) + C_4 \bar{B}_r(\mathcal{N}_m) + C_5 \bar{B}_p(\mathcal{N}_m) 
\end{align*}
where $C_1,C_2,C_3,C_4,C_5$ are the constants defined in the Lemma \ref{lemma1}.
\end{proof}

\begin{proof}[Proof of Theorem \ref{theorem1}]
    Note that the following relationship holds for the dynamic regret $\mathfrak{R}(T)$:
    \begin{align*}
        \mathfrak{R}(T)=\sum_{m=1}^{M} \bigg( \underbrace{\sum_{t \in \mathcal{G}_m} \left( V_t^* - V_t^{\pi_t} \right)}_{\text{Policy update error}} + \underbrace{\sum_{t \in \mathcal{N}_m} \left( V_t^* - V_t^{\pi_t} \right)}_{\text{Policy hold error}} \bigg).
    \end{align*}
    We use Lemma \ref{lemma1} to upper bound the use policy update error and use Lemma \ref{lemma2} to upper bound the policy hold error. This leads to 
    \begin{align*}
        \mathfrak{R}(T) &=\sum_{m=1}^{M} \bigg( \sum_{t \in \mathcal{G}_m} \left( V_t^* - V_t^{\pi_t} \right) + \sum_{t \in \mathcal{N}_m} \left( V_t^* - V_t^{\pi_t} \right) \bigg) \\
        & \leq \sum_{m=1}^{M} \bigg( \frac{C_1}{\eta \tau} \cdot \Big( 1- (1-\eta \tau)^{G_m} \Big) + G_m \Big( C_2 \delta^f_m + C_3\Big) + C_4 \bar{B}_r(\mathcal{G}_m) + C_5 \bar{B}_p(\mathcal{G}_m) \\
        \quad &+ N_m \cdot \Big( C_1(1-\eta \tau)^{G_m} + C_2 \delta_m^f + C_3 \Big) + C_4 \bar{B}_r(\mathcal{N}_m) + C_5 \bar{B}_p(\mathcal{N}_m) \bigg) \\
        &= \sum_{m=1}^{M} \bigg( \frac{C_1}{\eta \tau} + \left( N_m C_1 - \frac{C_1}{\eta \tau} \right) \left( 1- \eta \tau\right)^{G_m} + (N_m +G_m)(C_2 \delta_m^f +C_3) +\bar{B}(t_m,t_{m+1}) \bigg).
    \end{align*}
\end{proof}

\begin{proof}[Proof of Lemma \ref{lemma:R_pi}]
    The reader is referred to the proof of Theorem \ref{theorem_optimalGN}.
\end{proof}

\begin{proof}[Proof of Proposition \ref{proposition:R_env}]
    For fixed $t_{m},t_{m+1}$, note that $\bar{B}(t_m,t_{m+1})$ is a function of $G_m,N_m$ with the constraint $G_m+N_m = t_{m+1}-t_{m}$. In this proof, we let $\bar{B}(t_m,t_{m+1})$ to be denoted as a function $g(G_m,N_m)$. Recall that we have defined $\bar{B}(t_{m+1},t_{m}) := \bar{B}(\mathcal{N}_m) + \bar{B}(\mathcal{G}_m)$. Now, since $g(0,t_{m+1}-t_{m}) = g(t_{m+1}-t_m,0)=\sum_{t=t_m}^{t_{m+1}-1} \left(C_4 B_r(t_m,t) + C_5 B_p(t_m,t)\right)$, it is sufficient to show the existence of $G^{\dagger}_m \in (0,t_{m+1},t_m)$ and $N^{\dagger}_m \in (0,t_{m+1},t_m)$ that satisfy $g(G^{\dagger}_m,N^{\dagger}_m) < g(0,t_{m+1}-t_{m}) = g(t_{m+1}-t_m,0)$. By the definition of non-stationary environments (see Definition \ref{def:nonstationary}), let $t^{\dagger}_1,t^{\dagger}_2$ satisfy $B_r(t^\dagger_1,t^\dagger_2) >0 $ or $B_p(t^\dagger_1,t^\dagger_2) >0 $. Now, letting $G_m^\dagger = t^\dagger_2$, we have $B_r(t_m,G^\dagger_m)>0$ or $B_p(t_m,G^\dagger_m)>0$. As a result, either $\sum_{t_m}^{t_m+G^\dagger_m-1} B_r (t_m,t) + \sum_{t_m+G^\dagger_m}^{t_{m+1}-1} B_r (t_m+G^\dagger_m,t) < \sum_{t_m}^{t_{m+1}-1} B_r (t_m,t)$ or $\sum_{t_m}^{t_m+G^\dagger_m-1} B_p (t_m,t) + \sum_{t_m+G^\dagger_m}^{t_{m+1}-1} B_p (t_m+G^\dagger_m,t) < \sum_{t_m}^{t_{m+1}-1} B_p (t_m,t)$ holds. Now, by combining the two inequalities with the constants $C_4,C_5>0$ defined in Lemma \ref{lemma1}, we obtain that
    \begin{align*}
        &C_4 \bar{B}_r(t_m,t_m+G^\dagger_m) + C_4 \bar{B}_r(t_m+G^\dagger_m,T_{m+1}) + C_5 \bar{B}_p(t_m,t_m+G^\dagger_m) + C_5 \bar{B}_p(t_m+G^\dagger_m,t_{m+1}) \\
        \quad & < C_4  \bar{B}_r (t_m,t_{m+1}) + C_5  \bar{B}_p (t_m,t_{m+1})
    \end{align*}
    if and only if
    \begin{align*}
        \bar{B}(t_m,t_m+G^\dagger_m) + \bar{B}(t_m+G^\dagger_m,t_{m+1}) < \bar{B} (t_m,t_{m+1}).
    \end{align*}
    Therefore, $G^\dagger_m = t_2^\dagger, N^\dagger_m = t_{m+1}-t_m - t^\dagger_m$ satisfies the condition $g(G^{\dagger}_m,N^{\dagger}_m) < g(0,t_{m+1}-t_{m}) = g(t_{m+1}-t_m,0)$. This completes the proof.
\end{proof}

\begin{proof}[Proof of Theorem \ref{theorem_optimalGN}]
    We first show that the policy optimization error is a convex function of $G_m$ (or $N_m$. 
    Let $f_1(N_m,G_m) = C_1 (1- (1- \eta \tau)^{G_m}) + N_m C_1 (1- \eta \tau)^{G_m}$, where $N_m+G_m = t_{m+1} - t_m$ is a constant. Note that $\partial N_m / \partial G_m = -1$. It holds that
    \begin{align*}
        \frac{1}{C_1} \cdot \frac{\partial f_1}{\partial G_m} = \big\{  \ln (1 - \eta \tau) \left(N_m -1 \right) - 1 \big\} (1-\eta \tau)^{G_m}
    \end{align*}
    and 
    \begin{align*}
        \frac{1}{C_1} \cdot \frac{\partial^2 f_1}{\partial G_m^2} = \big\{  (\ln (1 - \eta \tau) )^2 (N_m-1) -2 \ln(1-\eta \tau)\big\} (1-\eta \tau)^{G_m}.
    \end{align*}
    Therefore, $\partial^2 f_1 / \partial G_m^2 >0 $ and  $\partial^2 f_1 / \partial N_m^2 >0 $ holds for $\forall N_m, G_m \geq 0$, where $N_m+G_m = t_{m+1}-t_m$ holds. The non-stationary terms are bounded as follows:
    \begin{align*}
        \bar{B}(\mathcal{N}_m) + \bar{B}(\mathcal{G}_m) &= (C_4 + C_5) \left( \bar{B}_r(\mathcal{N}_m) + \bar{B}_r(\mathcal{G}_m) \right).
    \end{align*}
    Note that by Assumption \ref{assum:Exponential order local variation budget}, $\bar{B}_r(\mathcal{N}_m) \leq \sum_{t=t_m+G_m}^{t=t_m+G_m+N_m-1} \alpha_r^{t-(t_m+G_m)}B^{\text{max}}(\mathcal{N}_m)$ and $\bar{B}_r(\mathcal{G}_m) \leq \sum_{t=t_m}^{t=t_m+G_m-1} \alpha_r^{t-t_m}B^{\text{max}}(\mathcal{G}_m)$. For the short notation, we use $\alpha_\diamond (\mathcal{G}_m) = \alpha_{\diamond,1}, \alpha_\diamond (\mathcal{N}_m) = \alpha_{\diamond,2}$ and $B_\diamond^{\text{max}}(\mathcal{G}_m) = B^{\text{max}}_{\diamond,1}, B^{\text{max}}_{\diamond}(\mathcal{N}_m) = B^{\text{max}}_{\diamond,2}$ where $\diamond = r~\text{or}~p$. Also, we let $\alpha_\square = \max (\alpha_{r,\square},\alpha_{p,\square})$ and $B^{\text{max}}_\square = \max(B^{\text{max}}_{r,\square}, B^{\text{max}}_{p,\square} )$, where $\square = 1~\text{or}~2$. One can write:
    \begin{align*}
        \bar{B}(\mathcal{N}_m) + \bar{B}(\mathcal{G}_m) &= C_4 \left( \bar{B}_r(\mathcal{N}_m) + \bar{B}_r(\mathcal{G}_m) \right) + C_5 \left( \bar{B}_p(\mathcal{N}_m) + \bar{B}_p(\mathcal{G}_m) \right) \\ 
        &\leq \left( C_4 +C_5 \right)\cdot \left( \frac{\alpha_1^{G_m}-1}{\alpha_1 -1} \cdot B_1^{\text{max}}+\frac{\alpha_{2}^{N_m}-1}{\alpha_{2} -1} \cdot B^{\text{max}}_2 \right).
    \end{align*}
    We denote the upper bound as a function $f_2 (N_m, G_m)$. Note that $B^{\text{max}}_1$ and $B_2^{\text{max}} >0$ hold for a non-stationary environment. If $0<\alpha_1,\alpha_2<1$, then $f_2 (N_m, G_m)$ is a concave function with respect to $(N_m,G_m)$. If $\alpha_1,\alpha_2>1$, then $f_2 (N_m, G_m)$ is a convex function with respect to $(N_m,G_m)$. 
\end{proof}

\section{Supplementary lemmas}
\begin{lemma}[NPG Convergence]
    Assume that we have an inexact Q value estimation at time $t_m +G_m-1$, $\hat{Q}_{t_m +G_m-1}$, where we denote $ Q_{t_m +G_m-1}$ as the exact Q value. Now, define the error of estimation as $\epsilon$, that is, $|| Q_{t_m +G_m-1} - \hat{Q}_{t_m +G_m-1} ||_\infty \leq \epsilon_f$. For any $g \in [G_m]$, it holds that 
    $$ V^*_{t_m + G_m-1} - V^{\pi^g}_{t_m + G_m-1} \leq (\gamma+2)((1- \eta \tau)^{g-1}C_1) + \frac{2(\gamma+2)}{1-\gamma}\left( 1 + \frac{\gamma}{\eta \tau} \right)\cdot \epsilon_f + \frac{2\tau \log{|\mathcal{A}|}}{1-\gamma}$$
    where 
$$C_1 = \left| \left| Q^*_\tau - Q^{(0)}_\tau \right| \right|_\infty + 2 \tau \left( 1- \frac{n\tau}{1-\gamma} \right)  \left| \left| \log{\pi_\tau^*} - \log{\pi^{(0)}} \right| \right|_\infty.$$
\label{lemma:NPGconvergence}
\end{lemma}

\begin{proof}[Proof of Lemma \ref{lemma:NPGconvergence}]
We omit the underscript $t$ for simplicity of notation, i.e., $V_{t},V_{\tau,t},V^*_{t}$ denote $V,V_\tau,V^*$, respectively. For any $m \in [M]$ and any $t \in \mathcal{G}_m$, the inequality
\begin{align*}
    V^*(s) - V(s) &\leq \left| \left| V^* (\cdot)- V_{\tau}^* (\cdot)\right| \right|_\infty + \left| \left| V^*_{\tau} (\cdot)- V_{\tau} (\cdot) \right| \right|_\infty  + \left| \left| V_{\tau} (\cdot)- V (\cdot) \right| \right|_\infty \nonumber \\ 
    & \leq \left| \left| V^*_{\tau} (\cdot)- V_{\tau} (\cdot) \right| \right|_\infty + \frac{2 \tau \log{|\mathcal{A}|}}{1-\gamma}
\end{align*}

holds since for any policy $\pi$, $||V^\pi_\tau - V^\pi||_\infty = \tau \max_{s} | \mathcal{H} (s,\pi)| \leq \frac{\tau \log |\mathcal{A}|}{1 - \gamma}$ holds.
Now, note that $V_{\tau}$ is a value function of a policy $\pi_{\tau}$ that we obtain after updating $g$ iterations. As a result,
\begin{align}
    \left| \left| V^*_{\tau} (\cdot)- V_{\tau} (\cdot) \right| \right|_\infty &\leq \tau \left| \left| \log{\pi^*_{\tau}} - \log{\pi^{g}_{\tau}} \right| \right|_\infty + \left| \left| Q^*_{\tau} (\cdot)- Q_{\tau} (\cdot) \right| \right|_\infty \nonumber \\
    & \leq \tau \cdot \frac{2}{\tau}\left( (1-\eta \tau)^{g-1}C_1 +C_2 \right) + \gamma \left( (1-\eta \tau)^{g-1}C_1 +C_2 \right) \label{lemma:NPGconvergence_eq1} \\ 
    & = (\gamma +2)\left( (1- \eta \tau)^{g-1}C_1 +C_2 \right) \nonumber
\end{align}
where 
$$C_1 = \left| \left| Q^*_\tau - Q^{(0)}_\tau \right| \right|_\infty + 2 \tau \left( 1- \frac{n\tau}{1-\gamma} \right)  \left| \left| \log{\pi_\tau^*} - \log{\pi^{(0)}} \right| \right|_\infty,~C_2 = \frac{2 \epsilon_f}{1-\gamma}\left( 1+ \frac{\gamma}{\eta\tau} \right).$$ The equation \eqref{lemma:NPGconvergence_eq1} holds cue to Theorem 2 in \cite{cen2022fast}.
\end{proof}

\begin{lemma}[Difference between optimal state action value functions of two MDPs]
    For any two time steps $t_1,t_2 \in T$, we denote the optimal Q functions at step $h \in [H]$  as $Q^{*}_{t_1,h}(s,a),Q^{*}_{t_2,h}(s,a)$. Then, for any state and action pair $s,a \in \mathcal{S} \times \mathcal{A}$,
    $$Q^{*}_{t_1,h}(s,a)-Q_{t_2,h}^{*}(s,a) \leq \sum_{h^\prime = h}^{H-1} \gamma^{h^\prime - h} B_r(t_1,t_2) + \frac{r_{\text{max}}}{1-\gamma} \sum_{h^\prime=h}^{H-1} \gamma^{h^\prime-h} B_p(t_1,t_2) $$
    holds, where $B_r(t_1,t_2)$ and $B_p(t_1,t_2)$ denote the local time-elapsing variation budgets between the time steps $\{t_1,t_1+1,t_1+2,...,t_2 \}$.
    \label{lemma:optimalQGap}
\end{lemma}

\begin{proof}[Proof of Lemma \ref{lemma:optimalQGap}]
    Only for the purpose of the proof of Lemma \ref{lemma:optimalQGap}, we define the state value function $V^{\pi}_{t,h} : \mathcal{S} \rightarrow \mathbb{R} $ and the state action value function $Q^{\pi}_{t,h} : \mathcal{S} \times \mathcal{A}  \rightarrow \mathbb{R} $ at step $h $ of time $t$ as 
    \begin{align*}
        V^{\pi}_{t,h}(s) := \mathbb{E}_{\mathcal{M}_t} \left[ \sum_{h^\prime=h}^{H-1} \gamma^{h^\prime-h} r_{t,h^\prime}~\bigg|~s^{0}_{t} = s \right]
    \end{align*}
    and
    \begin{align*}
        Q^{\pi}_{t,h}(s,a) \nonumber := \mathbb{E}_{\mathcal{M}_t} \left[ \sum_{h^\prime=h}^{H-1} \gamma^{h^\prime-h} r_{t,h^\prime} ~\bigg|~s^{0}_{t} = s,a^{0}_{t} = a \right].
    \end{align*}
    Note that the optimal state value function and the state action value function satisfy the following Bellman equation. 
    $$
    Q^*_{t,h}(s,a) = \left(R_{t,h} + \gamma P_{t}V^{*}_{t,h} \right)(s,a) , \pi^*_t = \argmax_{a \in \mathcal{A}} Q^{*}_{t,h} (s,a).
    $$
    The proof depends on a backward induction. First, the statement holds when $h=H-1$ since
    $$ \left| \left| Q^{*}_{t_1,H-1}(s,a) - Q^{*}_{t_2,H-1}(s,a) \right|\right|_\infty  = \left| \left| r_{t_1,H-1} - r_{t_2,H-1} \right| \right|_\infty = \left| \left| R_{t_1} - R_{t_2} \right| \right|_\infty.$$ 
    Now, we assume that the statement of Lemma \ref{lemma:optimalQGap} holds for $h+1$. Then, for $h$ it holds that
    \begin{align*}
       Q^{*}_{t_1,h}(s,a) - Q^{*}_{t_2,h}(s,a)  &= \left(R_{t_1,h} - R_{t_2,h} \right) (s,a) + \gamma \sum_{s^\prime \in \mathcal{S}} \bigg(P_{t_1}(s^\prime | s,a) V^*_{t_1,h+1}(s^\prime) - P_{t_2}(s^\prime | s,a) V^*_{t_2,h+1}(s^\prime) \bigg) \\ 
       &\leq B_r(t_1,t_2) + \gamma \sum_{s^\prime \in \mathcal{S}} \bigg(P_{t_1}(s^\prime | s,a) Q^*_{t_1,h+1}(s^\prime,\pi^*_{t_1}(s^\prime)) - P_{t_2}(s^\prime | s,a) Q^*_{t_2,h+1}(s^\prime,\pi^*_{t_2}(s^\prime)(s^\prime) \bigg).
    \end{align*}
    Then by the induction hypothesis on $h+1$, the following holds for any $s^\prime \in \mathcal{S}$:
    \begin{align*}
        Q^{*}_{t_1,h+1}(s^\prime,\pi^*_{t_1}(s^\prime)) &\leq Q^{*}_{t_2,h+1}(s^\prime,\pi^*_{t_1}(s^\prime)) + \sum_{h^\prime = h+1}^{H-1} \gamma^{h^\prime - (h+1)} B_r (t_1,t_2) + \frac{r_{\text{max}}}{1-\gamma} \sum_{h^\prime=h+1}^{H-1} \gamma^{h^\prime-(h+1)} B_p(t_1,t_2) \\ 
        &\leq Q^{*}_{t_2,h+1}(s^\prime,\pi^*_{t_2}(s^\prime)) + \sum_{h^\prime = h+1}^{H-1} \gamma^{h^\prime - (h+1)} B_r (t_1,t_2) + \frac{r_{\text{max}}}{1-\gamma} \sum_{h^\prime=h+1}^{H-1} \gamma^{h^\prime-(h+1)} B_p(t_1,t_2).
    \end{align*}
    Therefore, 
     \begin{align*}
       Q^{*}_{t_1,h}(s,a) - Q^{*}_{t_2,h^{\prime \prime}}(s,a) 
       &\leq B_r(t_1,t_2) + \gamma \sum_{s^\prime \in \mathcal{S}} \bigg( \Big( P_{t_1}(s^\prime | s,a) - P_{t_2}(s^\prime | s,a) \Big) Q^*_{t_2,h+1}(s^\prime,\pi^*_{t_2}(s^\prime) )\bigg) \\
       & \quad + \sum_{h^\prime = h+1}^{H-1} \gamma^{h^\prime - h} B_r (t_1,t_2) + \frac{r_{\text{max}}}{1-\gamma} \sum_{h^\prime=h+1}^{H-1} \gamma^{h^\prime-h} B_p(t_1,t_2) \\
       &\leq \gamma \left|\left| \Big( P_{t_1}(s^\prime | s,a) - P_{t_2}(s^\prime | s,a) \right| \right|_1 \left|\left| Q^*_{t_2,h+1}(s^\prime,\pi^*_{t_2}(s^\prime) ) \right| \right|_\infty + \sum_{h^\prime = h}^{H-1} \gamma^{h^\prime - h} B_r (t_1,t_2)  \\ 
       & \quad + \frac{r_{\text{max}}}{1-\gamma} \sum_{h^\prime=h+1}^{H-1} \gamma^{h^\prime-h} B_p(t_1,t_2) \\ 
       &\leq \gamma B_p(t_1,t_2) \cdot \frac{r_{max}}{1-\gamma} + \sum_{h^\prime = h}^{H-1} \gamma^{h^\prime - h} B_r (t_1,t_2) + \frac{r_{\text{max}}}{1-\gamma} \sum_{h^\prime=h+1}^{H-1} \gamma^{h^\prime-h} B_p(t_1,t_2) \\ 
       &=  \sum_{h^\prime = h}^{H-1} \gamma^{h^\prime - h} B_r(t_1,t_2) + \frac{r_{\text{max}}}{1-\gamma} \sum_{h^\prime=h}^{H-1} \gamma^{h^\prime-h} B_p(t_1,t_2).
    \end{align*}
    This completes the proof.
\end{proof}

\begin{corollary}[Difference between optimal state value functions of two MDPs] For any two times $t_1<t_2 \in T$, the gap between the two value functions at times $t_1$ and $t_2$ is bounded as 
$$ || V^*_{t_1}(s) - V^*_{t_2}(s) || _\infty \leq \frac{1-\gamma^H}{1-\gamma} \left( B_r(t_1,t_2) + \frac{r_{max}}{1-\gamma} B_p(t_1,t_2)\right).$$
\label{cor:optimalValueGap}
\end{corollary}

\begin{proof}[Proof of Corollary \ref{cor:optimalValueGap}]
    Corollary \ref{cor:optimalValueGap} comes from Lemma \ref{lemma:optimalQGap}.
\end{proof}

\begin{lemma}[Difference between value functions of two MDPs with same policy]
    For any two times $t_1,t_2 \in T$, any policy $\pi$, and any state $s \in \mathcal{S}$, the gap between the two value functions $V^\pi_{t_1}$ and $V^\pi_{t_2}$ is bounded as follows:
    $$V_{t_1}^{\pi}(s) - V_{t_2}^{\pi}(s) \leq \frac{1- \gamma^H}{1-\gamma} \cdot B_r(t_1,t_2) + \frac{\gamma}{1-\gamma} \cdot \left( \frac{1-\gamma^H}{1-\gamma} - \gamma^{H-1} H \right) \cdot B_p(t_1,t_2).$$
    \label{lemma:valuegap_samepolicy}
\end{lemma}

\begin{proof}[Proof of Lemma \ref{lemma:valuegap_samepolicy}]
    For a given initial state $s_0$, we first define the occupancy measure of state and action $(s,a)$ as 
    $$\rho_{t}^\pi (s,a) := \sum_{h=0}^{H-1} \gamma^h \mathbb{P}\left( s_h =s , a_h=a | P_t,\pi \right).$$
    It is worth noting that $\mathbb{P}\left( s_h =s , a_h=a | P_t,\pi \right) = \mathbb{P}\left( s_h =s| P_t,\pi \right) \cdot \pi(a_h=a | s_s =s)$. Now, note that the value function can be rewritten using the occupancy measure as 
    $$V^{\pi}_{t}(s) := \mathbb{E}_{\mathcal{M}_t} \left[ \sum_{h=0}^{H-1} \gamma^{h} r_{t,h}~|~s^{0}_{t} = s \right] = \mathbb{E}_{(s,a) \sim \rho^\pi_t} \left[ R_t(s,a) \right].$$
    Then for any $t_1,t_2 \in T$, the gap between the two value functions can be expressed as 
    \begin{align}
        V_{t_1}^{\pi}(s) - V_{t_2}^{\pi}(s) &=  \mathbb{E}_{(s,a) \sim d^\pi_{t_1}}\left[ R_{t_1}(s,a) \right] - \mathbb{E}_{(s,a) \sim d^\pi_{t_2}}\left[ R_{t_2}(s,a) \right] \nonumber \\ 
        &= \mathbb{E}_{(s,a) \sim d^\pi_{t_1}}\left[ R_{t_1}(s,a) - R_{t_2}(s,a) \right] - \mathbb{E}_{(s,a) \sim d^\pi_{t_1}}\left[ R_{t_2}(s,a) \right] + \mathbb{E}_{(s,a) \sim d^\pi_{t_2}}\left[ R_{t_2}(s,a) \right] \nonumber \\ 
        & \leq \frac{1- \gamma^H}{1-\gamma}\cdot \max_{(s,a)} \left(  \left| R_{t_2}(s,a) - R_{t_1}(s,a) \right| \right) + \left( \mathbb{E}_{(s,a) \sim d^\pi_{t_2}}\left[ R_{t_2}(s,a) \right] - \mathbb{E}_{(s,a) \sim d^\pi_{t_1}}\left[ R_{t_2}(s,a) \right] \right) \nonumber \\ 
        & = \frac{1- \gamma^H}{1-\gamma} \cdot B_r(t_1,t_2) + \left( \mathbb{E}_{(s,a) \sim d^\pi_{t_2}}\left[ R_{t_2}(s,a) \right] - \mathbb{E}_{(s,a) \sim d^\pi_{t_1}}\left[ R_{t_2}(s,a) \right] \right). \label{lemma_eq1}
    \end{align}
    
    Now, the gap $\mathbb{E}_{(s,a) \sim d^\pi_{t_2}}\left[ R_{t_2}(s,a) \right] - \mathbb{E}_{(s,a) \sim d^\pi_{t_1}}\left[ R_{t_2}(s,a) \right]$ is upper bounded as follows:
    \begin{align}
        \mathbb{E}_{(s,a) \sim d^\pi_{t_2}}\left[ R_{t_2}(s,a) \right] - \mathbb{E}_{(s,a) \sim d^\pi_{t_1}}\left[ R_{t_2}(s,a) \right] &\leq ||\rho^\pi_{t_2}(\cdot,\cdot) - \rho^\pi_{t_1}(\cdot,\cdot) ||_1 \cdot ||R_{t_2}(\cdot,\cdot)||_{\infty} \nonumber \\ 
        &= ||\rho^\pi_{t_2}(\cdot,\cdot) - \rho^\pi_{t_1}(\cdot,\cdot) ||_1 \cdot r_{\text{max}}. \label{lemma_eq2}
    \end{align}

    Now, the term $\sum_{(s,a)} | \rho^\pi_{t_2}(s,a) - \rho^\pi_{t_1}(s,a) |$ is bounded as follows:
    \begin{align}
        \sum_{(s,a)} \left| \rho^\pi_{t_2}(s,a) - \rho^\pi_{t_1}(s,a) \right| &= \sum_{(s,a)} \left|    \sum_{h=0}^{H-1} \Big( \gamma^h \cdot \big( \mathbb{P}\left( s_h =s| P_{t_2},\pi \right) - \mathbb{P}\left( s_h =s| P_{t_1},\pi \right) \big) \cdot \pi(a_h=a | s_h =s) \Big) \right| \nonumber \\ 
        &= \sum_{(s,a)} \left(    \sum_{h=0}^{H-1} \big| \gamma^h \cdot \big( \mathbb{P}\left( s_h =s| P_{t_2},\pi \right) - \mathbb{P}\left( s_h =s| P_{t_1},\pi \right) \big) \big| \cdot \big| \pi(a_h=a | s_h =s) \big| \right) \nonumber\\ 
        &= \sum_{s} \left(    \sum_{h=0}^{H-1} \big| \gamma^h \cdot \big( \mathbb{P}\left( s_h =s| P_{t_2},\pi \right) - \mathbb{P}\left( s_h =s| P_{t_1},\pi \right) \big) \big| \cdot \sum_{a \in \mathcal{A}} \big| \pi(a_h=a | s_h =s) \big| \right) \nonumber\\ 
        &= \sum_{s \in \mathcal{S}} \left(    \sum_{h=0}^{H-1} \big| \gamma^h \cdot \big( \mathbb{P}\left( s_h =s| P_{t_2},\pi \right) - \mathbb{P}\left( s_h =s| P_{t_1},\pi \right) \big) \big| \cdot 1 \right) \nonumber\\
        &= \sum_{h=0}^{H-1}  \gamma^h  \cdot \left( \sum_{s \in \mathcal{S}} \Big|  \big( \mathbb{P}\left( s_h =s| P_{t_2},\pi \right) - \mathbb{P}\left( s_h =s| P_{t_1},\pi \right) \big) \Big|\right). \label{lemma_eq3}
    \end{align}
Now, for simplicity of notation, we denote $\mathbb{P}(s_h=s | P_{t} , \pi)$ as $\mathbb{P}^h_t(s)$, $P_t(s_h=s | s_{h-1}=s^\prime,a_{h-1}=a^\prime)$ as $\mathbb{P}^h_t(s|s^\prime,a^\prime)$, and $\pi(a_h=a | s_{h}=s)$ as $\pi^h(a | s)$. Then, we have 

\begin{align*}
    &\sum_{s \in \mathcal{S}} \left| \mathbb{P}^h_{t_2}(s) - \mathbb{P}^h_{t_1}(s) \right| \\
    &= \sum_{s \in \mathcal{S}} \left| \sum_{s^\prime, a^\prime} \bigg( P^h_{t_2}(s | s^\prime,a^\prime) \cdot \pi^{h-1} (a^\prime | s^\prime ) \cdot \mathbb{P}^{h-1}_{t_2}(s^\prime)  - P^h_{t_1}(s | s^\prime,a^\prime) \cdot \pi^{h-1} (a^\prime | s^\prime ) \cdot \mathbb{P}^{h-1}_{t_1}(s^\prime) \bigg) \right| \\
    &\leq \sum_{s \in \mathcal{S}} \sum_{s^\prime, a^\prime} \left| \bigg( P^h_{t_2}(s | s^\prime,a^\prime) \cdot \mathbb{P}^{h-1}_{t_2}(s^\prime)  - P^h_{t_1}(s | s^\prime,a^\prime) \cdot \mathbb{P}^{h-1}_{t_1}(s^\prime) \bigg) \cdot \pi^{h-1} (a^\prime | s^\prime )  \right| \\
    &= \sum_{s^\prime, a^\prime}  \sum_{s \in \mathcal{S}}  \left| \bigg( P^h_{t_2}(s | s^\prime,a^\prime) \cdot \mathbb{P}^{h-1}_{t_2}(s^\prime)  - P^h_{t_1}(s | s^\prime,a^\prime) \cdot \mathbb{P}^{h-1}_{t_1}(s^\prime) \bigg) \cdot \pi^{h-1} (a^\prime | s^\prime )  \right| \\
    & \leq \sum_{s^\prime, a^\prime}  \sum_{s \in \mathcal{S}}   \bigg( \left| \left( P^h_{t_2}(s | s^\prime,a^\prime) - P^h_{t_1}(s | s^\prime,a^\prime) \right) \cdot \mathbb{P}^{h-1}_{t_2}(s^\prime) \cdot \pi^{h-1} (a^\prime | s^\prime )  \right| \\ 
    & \quad + \left| \left( \mathbb{P}^{h-1}_{t_2}(s^\prime) -  \mathbb{P}^{h-1}_{t_1}(s^\prime) \right)\cdot P^h_{t_1}(s | s^\prime,a^\prime) \cdot \pi^{h-1} (a^\prime | s^\prime )  \right| \bigg) \\ 
    & \leq \max_{s^\prime, a^\prime}\left(|| P^h_{t_2}(\cdot | s^\prime,a^\prime) - P^h_{t_1}(\cdot | s^\prime,a^\prime)||_1 \right) \cdot \left( \sum_{s^\prime, a^\prime} \left( \mathbb{P}^{h-1}_{t_2}(s^\prime) \cdot \pi^{h-1} (a^\prime | s^\prime ) \right) \right)  \\ 
     & \quad\quad   +  \left(\sum_{s^\prime, a^\prime} \left( \left| \mathbb{P}^{h-1}_{t_2}(s^\prime) -  \mathbb{P}^{h-1}_{t_1}(s^\prime) \right|  \right) \cdot \pi^{h-1} (a^\prime | s^\prime ) \right)  \cdot \left(\sum_{s \in \mathcal{S}} P^h_{t_1}(s | s^\prime,a^\prime) \right) \\
     & = B_p(t_1,t_2) \cdot \left( \sum_{s^\prime \in \mathcal{S}}  \mathbb{P}^{h-1}_{t_2}(s^\prime) \cdot \sum_{a^\prime \in \mathcal{A}} \pi^{h-1} (a^\prime | s^\prime )  \right) +  \left(\sum_{s^\prime \in \mathcal{S}} \left| \mathbb{P}^{h-1}_{t_2}(s^\prime) -  \mathbb{P}^{h-1}_{t_1}(s^\prime) \right|  \cdot \sum_{a^\prime \in \mathcal{A}} \pi^{h-1} (a^\prime | s^\prime ) \right)  \cdot 1 \\
     & = B_p(t_1,t_2) +  \sum_{s^\prime \in \mathcal{S}} \left| \mathbb{P}^{h-1}_{t_2}(s^\prime) -  \mathbb{P}^{h-1}_{t_1}(s^\prime) \right|.
\end{align*}

Now, note that $\sum_{s \in \mathcal{S}} \left| \mathbb{P}^0_{t_2}(s) - \mathbb{P}^0_{t_1}(s) \right|=0$ and  $\sum_{s \in \mathcal{S}} \left| \mathbb{P}^1_{t_2}(s) - \mathbb{P}^1_{t_1}(s) \right|=B_p(t_1,t_2)$ hold. Therefore,
$$\sum_{s \in \mathcal{S}} \left| \mathbb{P}^h_{t_2}(s) - \mathbb{P}^h_{t_1}(s) \right| \leq h B_p(t_1,t_2)$$ 
holds.
Then, substituting the above inequality into the inequality \eqref{lemma_eq3} gives that
\begin{align*}
    \sum_{(s,a)} \left| \rho^\pi_{t_2}(s,a) - \rho^\pi_{t_1}(s,a) \right|  &\leq \sum_{h=0}^{H-1}  \gamma^h h B_p(t_1,t_2) \\ 
     &\leq \frac{\gamma }{1-\gamma} \cdot \left( \frac{1-\gamma^H}{1-\gamma} - \gamma^{H-1} H \right) \cdot B_p(t_1,t_2).
\end{align*}
Now, it follows from the inequalities \eqref{lemma_eq1} and \eqref{lemma_eq2} that
$$V_{t_1}^{\pi}(s) - V_{t_2}^{\pi}(s) \leq \frac{1- \gamma^H}{1-\gamma} \cdot B_r(t_1,t_2) + \frac{\gamma}{1-\gamma} \cdot \left( \frac{1-\gamma^H}{1-\gamma} - \gamma^{H-1} H \right) \cdot B_p(t_1,t_2).$$
\end{proof}

\section{Experiment Platforms and Licenses}
\subsection{Platforms}
All experiments are conducted on 12 Intel Xeon CPU E5-2690 v4 and 2 Tesla V100 GPUs.

\subsection{Licenses}
We have used the following libraries/ repos for our Python codes: 
\begin{itemize}
    \item Pytorch (BSD 3-Clause ``New" or ``Revised" License).
    \item OpenAI Gym (MIT License).
    \item Numpy (BSD 3-Clause ``New" or ``Revised" License).
    \item Official codes distributed from \textit{https://github.com/pranz24/pytorch-soft-actor-critic}: to compare the performance of SAC and FSAC in the Mujoco environment.
    \item Official codes distributed from the \textit{https://github.com/linesd/tabular-methods}: to compare SAC and FSAC in the goal-switching cliff world.
\end{itemize}


\end{document}